%% file: main.tex
\documentclass[11pt]{article}
\usepackage[a4paper,top=3cm,bottom=2cm,left=3cm,right=3cm,marginparwidth=1.75cm]{geometry}
\usepackage{xcolor}
\usepackage[colorlinks = true,
            linkcolor = blue,
            urlcolor  = blue,
            citecolor = blue,
            anchorcolor = blue]{hyperref}
\usepackage[round]{natbib}

\usepackage{algorithm}
\usepackage{algorithmic}
\usepackage{dblfloatfix}
\usepackage[normalem]{ulem}
\usepackage[none]{hyphenat}
\usepackage{breakcites}
\usepackage[utf8]{inputenc} 
\usepackage[T1]{fontenc}    
\usepackage{hyperref}       
\usepackage{url}            
\usepackage{booktabs}       
\usepackage{amsfonts}       
\usepackage{nicefrac}       
\usepackage{microtype}      
\usepackage{xcolor}         
\usepackage{lipsum}
\usepackage{float}
\usepackage{caption}
\usepackage{subcaption}
\usepackage{mathtools}
\usepackage{soul}
\mathtoolsset{showonlyrefs}
\usepackage{enumitem}
\usepackage{amssymb}
\usepackage{amsthm}
\usepackage{tikz}
\usepackage{bm}
\usetikzlibrary{arrows}
\usepackage{dsfont}
\usepackage{graphicx, graphics}
\usepackage{wrapfig}
\usepackage{thmtools,thm-restate}

\usepackage[mathcal]{euscript}

\DeclareMathOperator*{\argmin}{arg\,min}

\newcommand{\textBlue}[1]{{\leavevmode\color{blue}#1}} 

\newcommand{\MSE}{\mathrm{MSE}}
\newcommand{\Var}{\mathrm{Var}}

\newcommand{\cH}{\mathcal{H}}

\newcommand{\cF}{\mathcal{F}}

\newcommand{\cX}{\mathcal{X}}

\newcommand{\cD}{\mathcal{D}}

\newcommand{\Indicator}{\mathds{1}}
\newcommand{\ev}{{\mathbb{E}}}
\newcommand{\pr}{\mathbb{P}}

\newcommand{\Ed}[2]{\ev_{#1}\left[#2\right]}

\newtheorem{definition}{Definition}

\title{Selective Regression Under Fairness Criteria}

\author{%
    Abhin Shah  \footnote{Equal contribution. Email address: \texttt{abhin@mit.edu, buyuheng@mit.edu}}\hspace{2mm}\textsuperscript{,}\footnote{Department of Electrical Engineering and Computer Science, Massachusetts Institute of Technology.}\\
	\and
	Yuheng Bu \footnotemark[1]\hspace{2mm}\textsuperscript{,}\footnotemark[2]\\
	\and
	Joshua Ka-Wing Lee\footnotemark[2]\\
	\and
	Subhro Das\footnote{MIT-IBM Watson AI Lab, IBM Research.}\\
	\and
	Rameswar Panda \footnotemark[3]\\
	\and
	Prasanna Sattigeri\footnotemark[3]\\
	\and
	Gregory W. Wornell\footnotemark[2]\\
}
\date{}
\begin{document}
\maketitle
\begin{abstract}
Selective regression allows abstention from prediction if the confidence to make an accurate prediction is not sufficient. In general, by allowing a reject option, one expects the performance of a regression model to increase at the cost of reducing coverage (i.e., by predicting on fewer samples). However, as we show, in some cases, the performance of a minority subgroup can decrease while we reduce the coverage, and thus selective regression can magnify disparities between different sensitive subgroups. Motivated by these disparities, we propose new fairness criteria for selective regression requiring the performance of every subgroup to improve with a decrease in coverage. We prove that if a feature representation satisfies the \textit{sufficiency} criterion or is \textit{calibrated for mean and variance}, then the proposed fairness criteria is met. Further, we introduce two approaches to mitigate the performance disparity across subgroups: (a) by regularizing an upper bound of conditional mutual information under a Gaussian assumption and (b) by regularizing a contrastive loss for conditional mean and conditional variance prediction. The effectiveness of these approaches is demonstrated on synthetic and real-world datasets.
\end{abstract}

\input{content/1introduction}
\input{content/2preliminaries}
\input{content/3problem_formulation}
\input{content/4algorithms}
\input{content/5experiments}
\input{content/6conclusion}
\subsection*{Acknowledgements}
This work was supported, in part, by the MIT-IBM Watson AI Lab, and its member companies Boston Scientific, Samsung, and Wells Fargo; and NSF under Grant No. CCF-1717610.

\bibliographystyle{abbrvnat}
\bibliography{references}

\onecolumn
\appendix
\section*{Appendix}
\input{content/7appendix}

\end{document}

%% file: content/1introduction.tex
 \section{Introduction}
\label{sec:intro}
As the adoption of machine learning (ML) based systems accelerates in a wide range of applications, including critical workflows such as healthcare management \citep{bellamy2018ai}, employment screening \citep{selbst2019fairness}, automated loan processing,
there is a renewed focus on the trustworthiness of such systems. An important attribute of a trustworthy ML system is to reliably estimate the uncertainty in its predictions.
For example, consider a loan approval ML system where the prediction task is to suggest appropriate loan terms (e.g., loan approval, interest rate). 
If the model's uncertainty in its prediction is high for an applicant, the prediction can be rejected to avoid potentially costly errors. The users of the system, i.e., the decision-maker, can intervene, and take remedial actions such as gathering more information for applicants with rejected model predictions or involving a special human credit committee before arriving at a decision. This paradigm is known as \textit{prediction with reject-option} or \textit{selective prediction}. 

By making the tolerance for uncertainty more stringent, the user expects the error rate of the predictions made by the system to decrease as the system makes predictions for fewer samples, i.e., as coverage is reduced.
Although the error may lessen over the entire population, \cite{jones2020selective} demonstrated that for \textit{selective classification}, this may not be true for different sub-populations.
In other words, selective classification could magnify disparities across different sensitive groups (e.g., race, gender). 
For example, in the loan approval ML system, the error rate for a sensitive group could increase with a decrease in coverage.
To mitigate such disparities, \cite{lee2021fair, schreuder2021classification} proposed methods for performing fair selective classification. 


In this work, we demonstrate and investigate the performance disparities across different subgroups for \textit{selective regression} as well as develop novel methods to mitigate such disparities. Similar to \cite{lee2021fair}, we do not assume access to the identity of sensitive groups at test time.
Compared to selective classification, one major challenge to tackling the aforementioned disparities in selective regression is as follows:
in selective classification, generating an \textit{uncertainty measure} (i.e., the model's uncertainty for its prediction) from an existing classifier is straightforward. For example, one could take the softmax output of an existing classifier as an uncertainty measure. In contrast, there is no direct method to extract an uncertainty measure from an existing regressor designed only to predict the conditional mean. \\

\noindent{\bf Contributions.} First, we show via the Insurance dataset and a toy example that selective regression, like selective classification, can decrease the performance of some subgroups when \emph{coverage} (the fraction of samples for which a decision is made) is reduced (see Section \ref{subsec:biases_sr}).
\begin{enumerate}[leftmargin=*,topsep=-0.6em,itemsep=-2pt]
\item Motivated by this, we provide a novel fairness criteria (Definition \ref{def:fair}) for selective regression, namely, \textit{monotonic selective risk}, which requires the risk of each subgroup to monotonically decrease with a decrease in coverage (see Section \ref{subsec:fair_sr}).
\item We prove that if a feature representation satisfies the standard \textit{sufficiency} criterion or is \textit{calibrated for mean and variance} (Definition \ref{def:cali}), then the \textit{monotonic selective risk} criteria is met (see Theorem \ref{thm:sufficiency} and \ref{thm:cali}).
\item We provide two neural network-based algorithms: one to impose the sufficiency criterion by regularizing an upper bound of conditional mutual information under a Gaussian assumption (see Section \ref{subsec:hetero}) and the other
to impose the calibration for mean and variance by regularizing a contrastive loss (see Section \ref{subsec:residue}). 
\item Finally, we empirically\footnote[1]{The source code is available at \url{github.com/Abhin02/fair-selective-regression}.} demonstrate the effectiveness of these algorithms on real-world datasets (see Section \ref{sec:expts}).
\end{enumerate}

%% file: content/2preliminaries.tex
\section{Background}
\label{sec:preliminaries}
\subsection{Fair Regression}
\label{subsec:fair_reg}
In standard (supervised) regression, given pairs of input variables $X\in\cX$ (e.g., demographic information) and target variable $Y\in \mathbb{R}$ (e.g., annual medical expenses), we want to find a predictor $f : \cX \rightarrow \mathbb{R}$ that best estimates the target variable for new input variables. 
Formally, given a set of predictors $\cF$ and a set of training samples of $X$ and $Y$, i.e., $\{(x_1,y_1),\dots,(x_n,y_n)\}$, the goal is to construct $f^* \in \cF$ which minimizes the mean-squared error (MSE):
\begin{equation}
    f^* = \argmin_{f \in \cF} \ev[(Y-f(X))^2]. \label{eq:supervised_goal}
\end{equation}

In fair regression, we augment the goal in \eqref{eq:supervised_goal} by requiring our predictor to retain ``fairness'' w.r.t.\ some sensitive attributes $D\in\cD$ (e.g., race, gender). For example, we may want our predictions of annual medical expenses using the demographic information not to  discriminate w.r.t.\ race. 

In this work, we assume $D$ to be discrete and consider members with the same value of $D$ as being in the same \textit{subgroup}. While numerous criteria have been proposed to enforce fairness in machine learning, we focus on the notion of \textit{subgroup risks} \citep{williamson2019fairness}, which ensures that the predictor $f$ behaves similarly (in terms of risks) across all subgroups. 
This notion, also known as accuracy disparity, has been used frequently in fair regression, e.g., \cite{chzhen2020example, chi2021understanding}, and has also received attention in the field of domain generalization, e.g., \cite{krueger2021out}.
Formally, given a set of training samples of $X$, $Y$, and $D$, i.e., $\{(x_1,y_1,d_1),\dots,(x_n,y_n,d_n)\}$, the goal is to construct $f^* \in \cF$ which minimizes the overall MSE subject to the subgroup MSE being equal for all subgroups:
\begin{align}
    & f^* = \argmin_{f \in \cF} \ev[(Y-f(X))^2] \text{~~s.t~~} \forall d,d' \in \cD,\\
    & \ev[(Y\hspace{-0.5mm}-\hspace{-0.5mm}f(X))^2|D = d] = \ev[(Y\hspace{-0.5mm}-\hspace{-0.5mm}f(X))^2|D = d']. \label{eq:fair_goal}
\end{align}
In this work, we consider the scenario where the sensitive attribute is available only during training i.e., we do not assume access to the sensitive attribute at test time.
\subsection{Selective Regression}
\label{subsec:sel_reg}
In selective regression, given pairs of input variables $X\in\cX$ and target variable $Y\in \mathbb{R}$, for new input variables, the system has a choice to: (a) make a prediction of the target variable or (b) abstain from a prediction (if it is not sufficiently confident). In the example of predicting annual medical expenses, we may prefer abstention in certain scenarios to avoid harms arising from wrong predictions. By only making predictions for those input variables with low prediction uncertainty, the performance (in terms of MSE) is expected to improve. Formally, in addition to a predictor $f : \cX \rightarrow \mathbb{R}$ that best estimates the target variable for new input variables, we need to construct a rejection rule $\Gamma : \cX \rightarrow \{0,1\}$ that decides whether or not to make a prediction for new input variables. Thereafter, for $X = x$, the system outputs $f(x)$ as the prediction when $\Gamma(x)=1$, and makes no prediction if $\Gamma(x)=0$. 

There are two quantities that characterize the performance of selective regression: (i) \textit{coverage}, i.e., the fraction of samples that the system makes predictions on, which is denoted by $c(\Gamma) = \pr(|\Gamma(X)| = 1)$ and (ii) the MSE when prediction is performed
$$\mathbb{E}[(Y-f(X))^2|\Gamma(X)=1].$$
In order to construct a rejection rule $\Gamma$, we need some measure of the uncertainty $g(\cdot)$ associated with a prediction $f(\cdot)$. Then, the rejection rule $\Gamma$ can be defined as:
\begin{equation}\label{eq:rejection_rule}
    \Gamma(x) \coloneqq \begin{cases}
		1, & \text{if $g(x) \le \tau$}\\
        0, & \text{otherwise.}
	 \end{cases}
\end{equation}
where $\tau$ is the parameter that balances the MSE vs. coverage tradeoff: larger $\tau$ results in a larger coverage but also yields a larger MSE. Therefore, $\tau$ can be interpreted as the cost for not making predictions. As discussed in \cite{zaoui2020regression}, a natural choice for the uncertainty measure $g(\cdot)$ could be the conditional variance of $Y$ given $X$.

The goal of selective regression is to build a model with (a) high coverage and (b) low MSE. However, there may not exist any $\tau$ for which both (a) and (b) are satisfied simultaneously. Therefore, in practice, the entire MSE vs. coverage tradeoff curve is generated by sweeping over all possible values of $\tau$ allowing the system designer to choose any convenient operating point.
\subsection{Related Work}
{\bf Selective Regression.} While selective classification has received a lot of attention in the machine learning community \cite{chow1957optimum, chow1970optimum, hellman1970nearest, herbei2006classification, bartlett2008classification, nadeem2009accuracy, lei2014classification, geifman2017selective}, there is very limited work on selective regression. It is also known that existing methods for selective classification cannot be used directly for selective regression \cite{jiang2020risk}.
\cite{wiener2012pointwise} studied regression with reject option to make predictions inside a ball of small radius around the regression function with high probability. \cite{geifman2019selectivenet} proposed SelectiveNet, a neural network with an integrated reject option, to optimize both classification (or regresssion) performance and rejection rate simultaneously.
\cite{zaoui2020regression} considered selective regression with a fixed rejection rate and derived the optimal rule which relies on thresholding the conditional variance.
%
\cite{jiang2020risk} analyzed selective regression with the goal to minimize the rejection rate given a regression risk bound. We emphasize that none of these works study the question of fairness in selective regression.\\

\noindent{\bf Fair Regression.} \cite{calders2013controlling}, one of the first works on fair regression, studied linear regression with constraints on the mean outcome or residuals of the models. More recently, several works including \cite{berk2017convex, perez2017fair, komiyama2017two, komiyama2018nonconvex, fitzsimons2018equality, raff2018fair, agarwal2019fair, nabi2019learning,  oneto2020general} considered fair regression employing various fairness criteria. \cite{mary2019fairness} and \cite{lee2020maximal}  enforced independence between prediction and sensitive attribute by ensuring that the maximal correlation is below a fixed threshold. \cite{chzhen2020fair} considered learning an optimal regressor requiring the distribution of the output to remain the same conditioned on the sensitive attribute.
We emphasize that none of these works could be used for selective regression as they are designed to only predict the conditional mean and not the conditional variance (i.e., the uncertainty).\\

\noindent{\bf Fairness Criteria.} Numerous metrics and criteria have been proposed to enforce fairness in machine learning \cite{verma2018fairness}. Many of these criteria are mutually exclusive outside of trivial cases \citep{golz2019paradoxes}. Further, the existing approaches also differ in the way they enforce these criteria: (a) pre-processing methods \citep{zemel2013learning, louizos2015variational, calmon2017optimized} modify the training set to ensure fairness of any learned model, (b) post-processing methods \cite{hardt2016equality, pleiss2017fairness, corbett2017algorithmic} transform the predictions of a trained model to satisfy a measure of fairness, and (c) in-processing methods \cite{kamishima2011fairness, zafar2017fairness, agarwal2018reductions} modify the training process to directly learn fair predictors e.g., minimizing a loss function that accounts for both accuracy and fairness as in \eqref{eq:fair_goal}. Additionally, these criteria also differ in the kind of fairness they consider (see \cite{mehrabi2021survey, castelnovo2022clarification} for details): (a) group fairness 
ensures that subgroups that differ by 
sensitive attributes are treated similarly and
(b) individual fairness 
ensures that individuals who are similar (with respect to some metric) are treated similarly. In this work, we consider group fairness and propose a novel fairness criteria specific to selective regression. Our approach falls under the umbrella of in-processing methods as will be evident in Section \ref{sec:algo_design} (see \eqref{equ:feature_update}, \eqref{eq:in1}, and \eqref{eq:in2}).

%% file: content/3problem_formulation.tex
\section{Fair Selective Regression}
\label{sec:prob_formulation}
While fair regression and selective regression have been independently studied before, consideration of fair selective regression (i.e., selective regression while ensuring fairness) is missing in the literature.
In this section, we explore the disparities between different subgroups that may arise in selective regression. Building on top of this, we formulate a notion of fairness for selective regression.

\subsection{Disparities in Selective Regression}
\label{subsec:biases_sr}
\cite{jones2020selective} argued that, in selective classification, increasing abstentions (i.e., decreasing coverage) could \textit{decrease} accuracies on some subgroups and observed this behavior for CelebA dataset. In this section, we show that a similar phenomenon can be observed in selective regression.\\

\begin{figure*}[!t]
\begin{subfigure}{.3\textwidth}
  \centering
  \captionsetup{justification=centering}
  \includegraphics[width=0.95\textwidth]{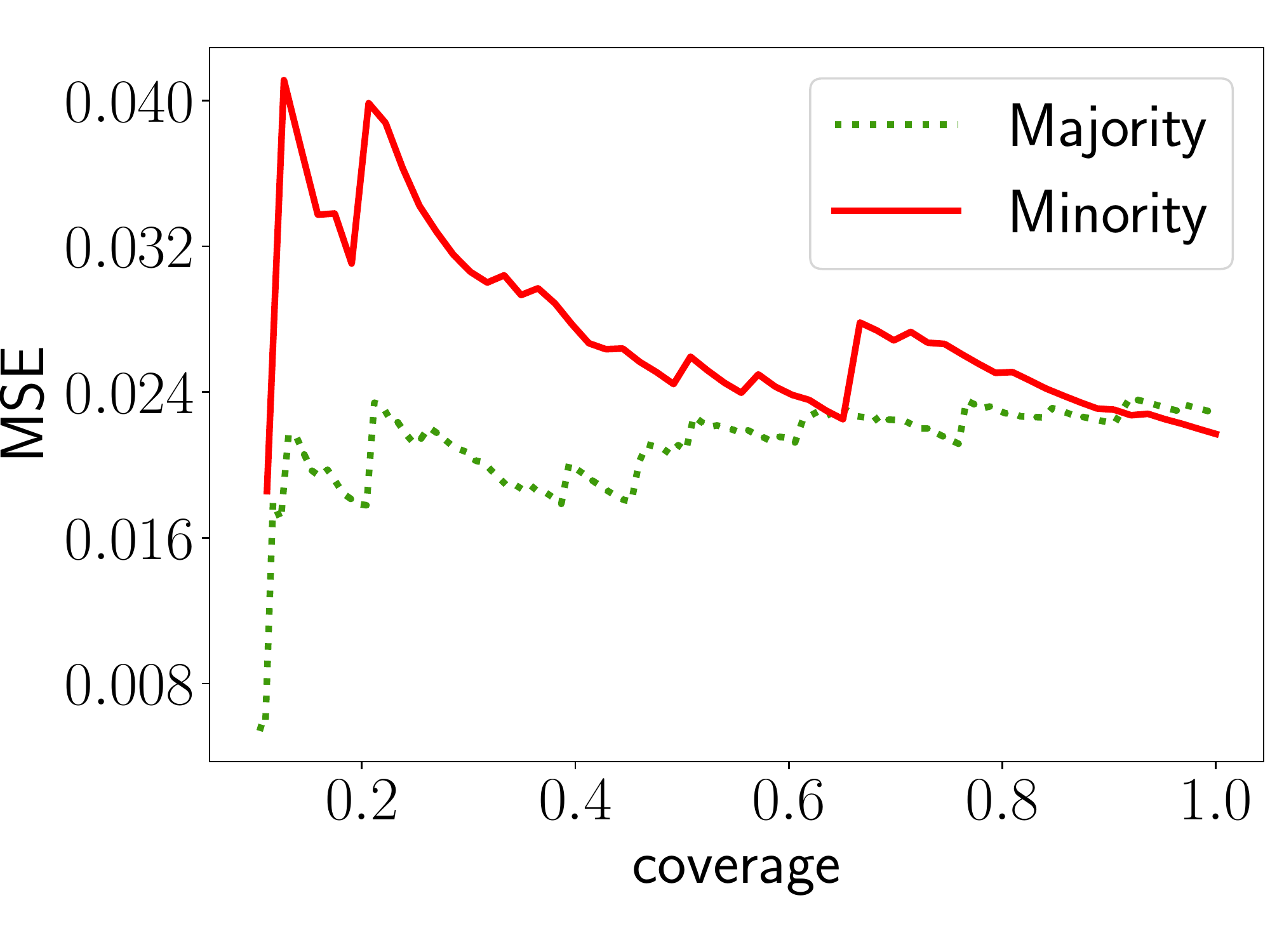}
\caption{The disparity for the Insurance dataset
when the predictor is the conditional
expectation \& the uncertainty measure
is the conditional variance.}
  \label{fig:disparities_insurance}
\end{subfigure}
\begin{subfigure}{.3\textwidth}
  \centering
  \captionsetup{justification=centering}
  \includegraphics[width=0.95\textwidth]{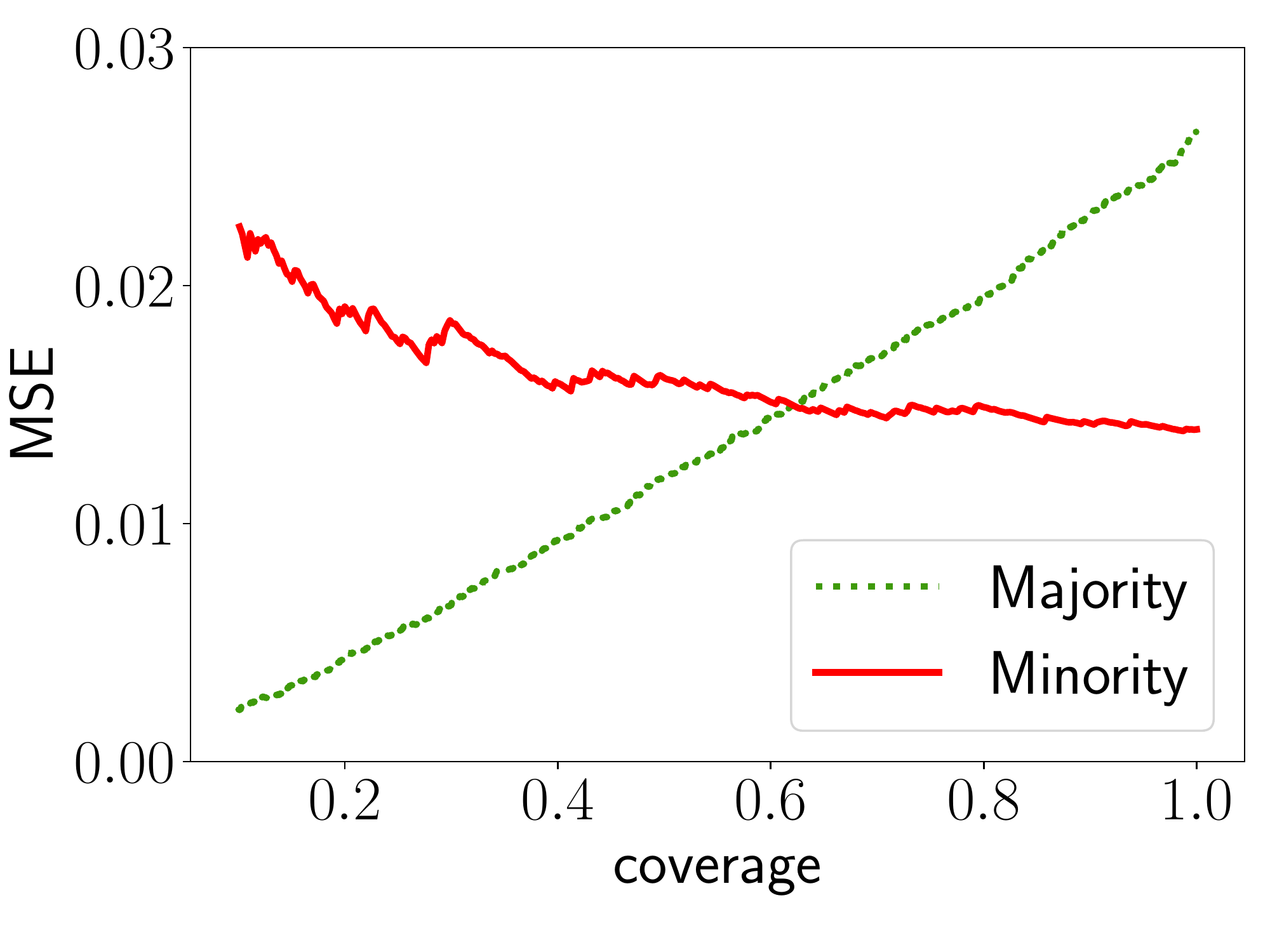}
\caption{The disparity for toy example when the
predictor is the conditional expectation \&
the uncertainty measure is the conditional
variance of $Y$ given $X_1$ and $X_2$.}
  \label{fig:synthetic_baseline}
\end{subfigure}
\begin{subfigure}{.32\textwidth}
  \centering
  \captionsetup{justification=centering}
  \includegraphics[width=0.95\textwidth]{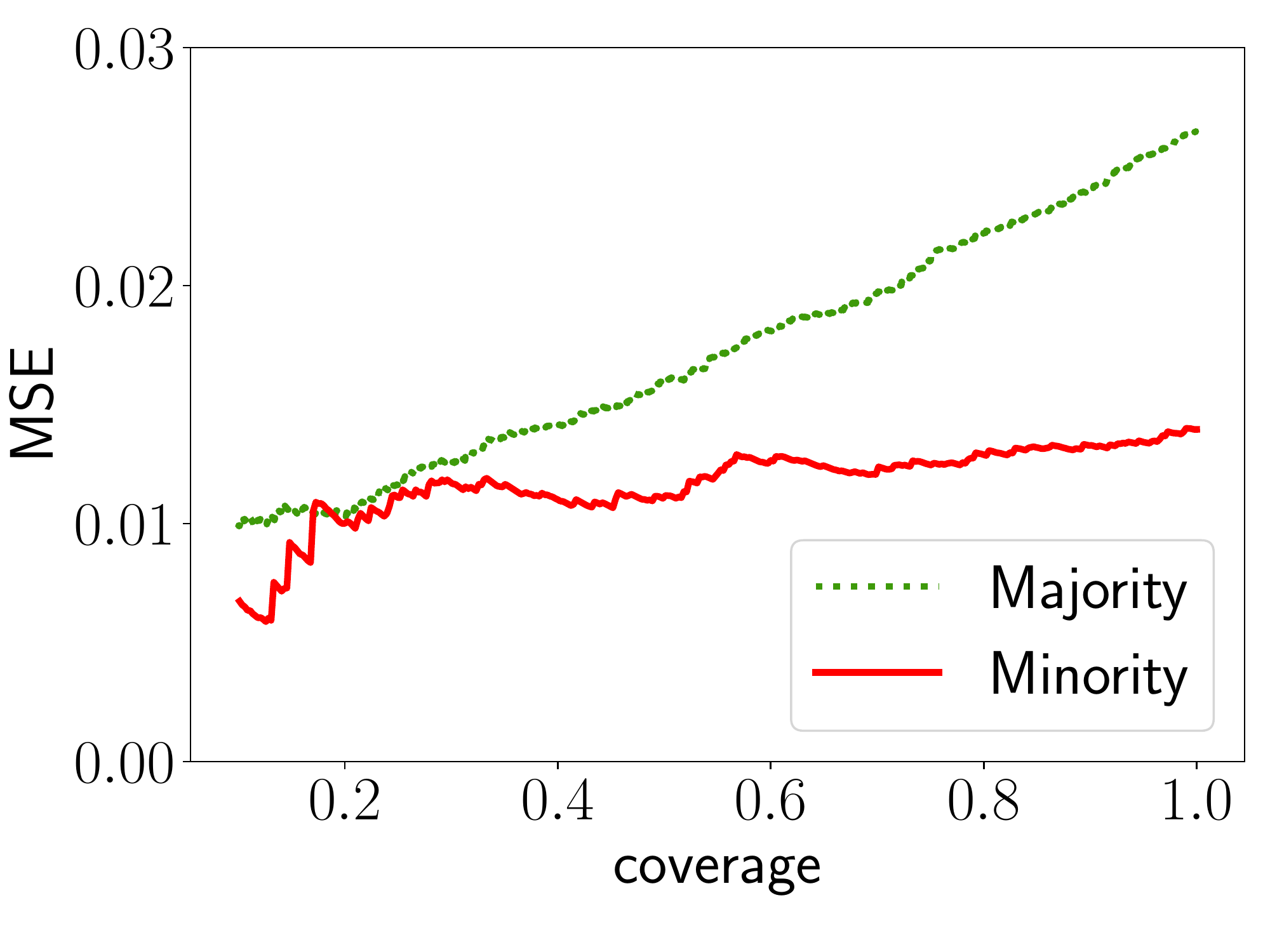}
\caption{The disparity could be mitigated for toy
example when the predictor is the conditional
expectation \& the uncertainty measure
is the conditional variance of $Y$ given $X_1$.}
  \label{fig:synthetic_sufficiency}
\end{subfigure}
\caption{Subgroup MSE vs. coverage to illustrate disparities in selective regression via Insurance dataset and toy example.}
\label{fig:app_toy_expts}
\vspace{-0.8em}
\end{figure*}

\noindent{\bf Insurance dataset.} Consider the task of predicting the annual medical expenses charged to patients from input variables such as age, BMI, number of children, etc., as in Insurance dataset. Suppose we construct our predictor as the conditional expectation and our uncertainty measure as the conditional variance following \cite{zaoui2020regression}. Then, generating the subgroup MSE vs. coverage tradeoff curve\footnote[2]{We have used \texttt{Baseline 1} to generate this (see Section \ref{sec:expts}).} across the subgroups induced by gender, as shown in Figure \ref{fig:disparities_insurance}, we see that while decreasing the coverage improves the performance for the majority subgroup (i.e., females), the performance for the minority subgroup (i.e., males) degrades. Thus, unknowingly, selective regression can magnify disparities between different sensitive subgroups.
Motivated to further understand this phenomenon, we explicitly recreate it via the following toy example.\\

\noindent{\bf Toy example.} Consider predicting $Y$ from two normalized input variables $X_1$ and $X_2$ that are generated i.i.d. from the uniform distribution over $[0,1]$. Suppose we have a binary sensitive attribute $D$ with $\pr(D\!=\!0)=0.9$, where $D = 0$ represents majority and $D = 1$ represents minority. 
To illustrate the disparities that may arise in selective regression, we let the distribution of $Y$ differ with $D$.
More specifically, for majority subgroup, we let the target be
$$
Y|_{D=0}= X_1+X_2+\mathcal{N}(0,0.1 X_1+0.15 X_2)
$$
and, for minority subgroup, we let the target be 
$$
Y|_{D=1}= X_1+X_2+\mathcal{N}(0,0.1 X_1+0.15 (1-X_2)).
$$
To summarize, the only difference is that for the majority, the variance of $Y$ increases in $X_2$, and for the minority, the variance of $Y$ decreases in $X_2$. 

In this case, the conditional variance $\Var(Y|X)$, i.e., our uncertainty measure would mainly capture the behavior of the majority subgroup $D = 0$, i.e., the subgroup with more samples. Since $\Var(Y|X,D\!=\!0)$ differs significantly from $\Var(Y|X,D\!=\!1)$, $\Var(Y|X)$ may not be a good measure of uncertainty for the minority subgroup $D = 1$. As a result, for minority subgroup, when we decrease the coverage, we may make predictions on samples that erroneously achieve low uncertainty based on $\Var(Y|X)$. This results in an increase in the MSE for the subgroup $D = 1$ (Figure \ref{fig:synthetic_baseline}).

An alternative could be to use the conditional variance of $Y$ given $X_1$ (instead of both $X_1$ and $X_2$) as our uncertainty measure. While this may be a slightly worse measure of uncertainty for $D\!=\!0$ than $\Var(Y|X)$, it is a much better measure of uncertainty for $D\!=\!1$. All in all, using this uncertainty measure, when we decrease the coverage, we make predictions on samples with low uncertainty for all subgroups resulting in a decrease in the MSE for every subgroup as shown in Figure \ref{fig:synthetic_sufficiency} (albeit at the cost of a slight increase in the overall MSE).

It is important to note that the toy example is designed to highlight the disparities when the uncertainty measure, a component of selective regression, is designed unfairly. The disparities could generally occur due to the predictor or the uncertainty measure (or both).
\subsection{When is Selective Regression Fair?}
\label{subsec:fair_sr}
Motivated by the disparities (that may arise) in selective regression, as shown above, we define the first notion of fairness for selective regression, which is called \textit{monotonic selective risk}. This notion requires our predictor and uncertainty measure to ensure the primary goal of selective regression that the subgroup MSE decreases monotonically with a decrease in coverage for every subgroup. The subgroup MSE for $d \in \cD$, as a function of the predictor $f$ and the uncertainty measure $g$, for a fixed coverage (parameterized by $\tau$) is given by
$$\MSE(f, g, \tau, d) =  \ev[(Y-f(X))^2|g(X) \le \tau, D=d].$$
Now, we formalize the criteria of \textit{monotonic selective risk} which ensures that no subgroup is discriminated against when the coverage is reduced in selective regression.
\begin{definition}\label{def:fair}
We say that a predictor $f$ and an uncertainty measure $g$ satisfy \textit{monotonic selective risk} if for any $\tau < \tau'$
$$\MSE(f, g, \tau, d) \leq \MSE(f, g, \tau', d) ~~ \forall d \in \cD.$$
\end{definition}
Inspired by the success of representation based learning in machine learning \cite{bengio2013representation}, we seek to find a representation $\Phi : \cX \rightarrow \cH$ that maps the input variable $X \in \cX$ to an intermediate representation $\Phi(X) \in \cH$ 
and use $\Phi(X)$ to construct our predictor $f : \cH \rightarrow \mathbb{R}$ and our uncertainty measure $g : \cH \rightarrow \mathbb{R}_{\geq 0}$. Then, for $X = x$, our prediction and uncertainty measure are (with a slight abuse of notation) $f(\Phi(x))$ and $g(\Phi(x))$ respectively.
\section{Theoretical Results}
In this section, we show that under certain conditions on the feature representation $\Phi$, the conditional mean as the predictor and the conditional variance as the uncertainty measure satisfy \textit{monotonic selective risk} (Definition \ref{def:fair}). 
\subsection{Sufficiency}
The \textit{sufficiency} criterion requires\footnote[3]{While conventionally sufficiency requires $Y\!\perp\!D | \ev[\Phi(X)]$, this notion has been adapted to incorporate feature representation/score function (e.g., Sec 1.1 in \cite{liu2019implicit}), Sec 3.4 in \cite{castelnovo2022clarification}.} $Y\!\perp\!D | \Phi(X)$, i.e., the learned representation $\Phi(X)$ completely subsumes all information about the sensitive attribute that is relevant to the target variable \citep{cleary1966test}. 
Sufficiency is closely tied with learning domain-invariant feature representation \citep{arjovsky2019invariant, creager2020exchanging} and has been used in fair selective classification \citep{lee2021fair}.

The theorem below shows that if the feature representation is sufficient, then the choice of conditional mean as the predictor and the conditional variance as the uncertainty measure ensures the fairness criteria of \textit{monotonic selective risk}, i.e.,  the subgroup MSE decreases monotonically with coverage for all subgroups.
See Appendix \ref{appendix:proof_theorem} for a proof.
\begin{restatable}{theorem}{thmSufficiency}
\label{thm:sufficiency}
Suppose the representation $\Phi(X)$ is sufficient i.e., $Y\!\perp\!D | \Phi(X)$. Let $f(\Phi(X)) = \ev[Y | \Phi(X)]$ and $g(\Phi(X)) = \Var[Y | \Phi(X)]$. Then, for any $d \in \cD$ and any $\tau < \tau'$, we have $\MSE(f, g, \tau, d) < \MSE(f, g, \tau', d)$.
\end{restatable}
\subsection{Calibration for mean and variance}
\label{subsec:calibration}
In practice, the conditional independence $Y\!\perp\!D | \Phi(X)$ required by sufficiency  may be too difficult to satisfy. Since we only care about MSE, which depends on the first and second-order moments,
one could relax the sufficiency condition by requiring the representation $\Phi$ to be such that these moments are the same across all subgroups. This inspires our notion of $\Phi$ \textit{calibrated for mean and variance}.  
\begin{definition}\label{def:cali}
We say a representation $\Phi(X)$ is \textit{calibrated for mean and variance} if
\begin{align}
    \ev[Y | \Phi(X), d] & = \ev[Y | \Phi(X)] ~~\forall d \in \cD, \label{eq:cali_mean} \\
    \qquad \Var[Y | \Phi(X), d] & = \Var[Y | \Phi(X)] ~~\forall d \in \cD. \label{eq:cali_var}
\end{align}
\end{definition}


The theorem below shows that if the feature representation is \textit{calibrated for mean and variance}, then the choice of conditional mean as the predictor and the conditional variance as the uncertainty measure ensures the fairness criteria of \textit{monotonic selective risk}, i.e.,  the subgroup MSE decreases monotonically with coverage for all subgroups. The proof is similar to the proof of Theorem \ref{thm:sufficiency} and is omitted.
\begin{restatable}{theorem}{thmCali}
\label{thm:cali}
Suppose the representation $\Phi(X)$ is calibrated for mean and variance. Let $f(\Phi(X)) = \ev[Y | \Phi(X)]$ and $g(\Phi(X)) = \Var[Y | \Phi(X)]$. Then, for any $d \in \cD$ and any $\tau < \tau'$, we have $\MSE(f, g, \tau, d) < \MSE(f, g, \tau', d)$.
\end{restatable}

%% file: content/4algorithms.tex
\section{Algorithm Design}
\label{sec:algo_design}
In this section, we provide two neural network-based algorithms: one to impose sufficiency and the other to impose the calibration for mean and variance.

\subsection{Imposing sufficiency}
\label{subsec:hetero}
To simplify our algorithm when directly enforcing sufficiency, we utilize the framework of heteroskedastic neural network \cite{gal2016uncertainty}.
A heteroskedastic neural network, which requires training only a single neural network, assumes that the distribution of $Y$ conditioned on $X$ is Gaussian. Then, it is trained by minimizing the negative log likelihood:
$$L_G(\Phi, \theta) \triangleq - \sum_{i=1}^n \log \pr_G(y_i|\Phi(x_i);\theta)$$
where $\pr_G(y|\Phi(x) ;\theta)$ represents a Gaussian distribution with $f (\Phi(x); \theta_f)$ and $g(\Phi(x);\theta_g)$ as the conditional mean and the conditional variance (of $Y$ given $\Phi(X)$) respectively. The feature representation $\Phi$ is parameterized by $\theta_{\Phi}$ and the neural network is supposed to learn the parameters $\theta_{\Phi}$ and $\theta = (\theta_f, \theta_g)$.

\noindent To impose sufficiency, we augment minimizing the negative log likelihood as follows:
$$\min_{\theta,\Phi} \quad  L_G(\Phi, \theta),\quad
    \mathrm{s.t.} \quad Y \perp D | \Phi(X).$$
To relax the hard constraint of $Y \perp D | \Phi(X)$ into a soft constraint, we use the conditional mutual information, since $Y \perp D | \Phi(X)$ is equivalent to $I(Y;D|\Phi(X)) = 0$. For $\lambda \geq 0$,
$$\min_{\theta,\Phi} L_G(\Phi, \theta) + \lambda I(Y;D|\Phi(X)).$$
As discussed in \cite{lee2021fair}, existing methods using mutual information for fairness are ill-equipped to handle conditioning on the feature representation $\Phi(\cdot)$. Therefore, we further relax the soft constraint by using the following upper bound for $I(Y;D|\Phi(X))$ from \cite{lee2021fair}:
\begin{align}
\label{bound}
I(Y;D|\Phi(X))\le \Ed{\Phi(X),Y,D}{\log \pr (Y|\Phi(X),D)} - \Ed{D}{ \Ed{\Phi(X),Y}{\log \pr (Y|\Phi(X),{D})}}.
\end{align}
where equality is achieved if and only if $Y \perp D \mid \Phi(X)$.

In order to compute the upper bound in \eqref{bound}, we need to learn the unknown distribution $\pr(y|\Phi(x),d)$. We approximate this by $\pr_G(y|\Phi(x),d ;w)$ which is a Gaussian distribution with $f (\Phi(x),d; w_f)$ and $g(\Phi(x),d;w_g)$ as the conditional mean and the conditional variance of $Y$ given $\Phi(X)$ and $D$, respectively. The neural network is supposed to learn the parameters $w = (w_f, w_g)$.

\begin{algorithm}[tb]
\caption{Heteroskedastic neural network with sufficiency-based regularizer}
\label{alg:suff}
\begin{algorithmic}
    \STATE {\bfseries Input:} training samples $\{(x_i,y_i,d_i)\}_{i=1}^{n}$, regularizer $\lambda$
    \STATE {\bfseries Draw:} $\{\widetilde{d}_1,\dots,\widetilde{d}_n\}$ drawn i.i.d. from $\hat{\pr}_D$
    \STATE {\bfseries Initialize:} $\theta$, $\theta_{\Phi}$, and $w^{(d)}$ with pre-trained models 
    \STATE {\bfseries Initialize:} $n_d = $ number of samples in group $d$, $\forall d \in \cD$
\FOR{each training iteration}
    \FOR{each batch}
        \FOR{$d=1,\dots,|\cD|$}
        \STATE{\hspace{-3mm}\textBlue{\# update subgroup-specific mean/variance predictor}}
        \STATE {
        $w^{(d)} \leftarrow w^{(d)} - \frac{1}{n_d} \eta_{w} \nabla_{w} L_d(w)$}
        \ENDFOR
    \ENDFOR
    \FOR{each batch}
    \STATE{\textBlue{\# update feature extractor}}
    \STATE {$\theta_{\Phi} \leftarrow \theta_{\Phi} - \frac{1}{n} \eta \nabla_{\theta_{\Phi}} (L_G(\Phi, \theta)+\lambda L_R(\Phi))$}  
    \STATE{\textBlue{\# update mean/variance predictor}}
    \STATE {$\theta \leftarrow \theta - \frac{1}{n} \eta \nabla_{\theta} L_G(\Phi, \theta)$}
    \ENDFOR
\ENDFOR
\end{algorithmic}
\end{algorithm}
In scenarios where $\Phi(X)$ is high-dimensional compared to $D$, it would be preferred to approximate $\pr_G(y|\Phi(x),d ;w)$ by $\pr_G(y|\Phi(x);w^{(d)})$, i.e., train a \textit{subgroup-specific} Gaussian model with parameters $w^{(d)}$ for each $d\in \cD$ instead of using $D$ as a separate input to ensure that $D$ has an effect in $\pr_G(y|\Phi(x),d ;w)$. Then, for $d \in \cD$,
\begin{align}\label{equ:fitting}
    w^{(d)} = \argmin_{w}  L_d(w), \qquad \text{where}\\
    L_d(w) \triangleq -\sum_{i:\ d_i=d}{\log \pr_G (y_i|\Phi(x_i);w)}. 
\end{align}

To summarize, the first term of the upper bound in \eqref{bound} is approximated by the log-likelihood of the training samples using $\pr_G(y|\Phi(x);w^{(d_i)})$ for each subgroup $d_i \in \cD$ (i.e., subgroup-specific loss). 
Then, drawing $\widetilde{d}_i$ i.i.d from ${\pr}_D$ i.e., the marginal distribution of $D$ (which could be approximated by $\hat{\pr}_D$), the second term of the upper bound in \eqref{bound} is approximated by the negative log-likelihood of the samples using the  randomly-selected Gaussian model $\pr_G(y|\Phi(x);w^{(\widetilde{d}_i)})$ for each subgroup $\widetilde{d}_i \in \cD$ (i.e., subgroup-agnostic loss).
Combining everything and replacing all the expectations in 
\eqref{bound} with empirical averages, the regularizer is given by
\begin{equation}\label{equ:reg}
    L_R(\Phi) \triangleq \sum_{i=1}^n\log \Big(\frac{\pr_G  (y_i|\Phi(x_i);w^{(d_i)})}{ \pr_G (y_i|\Phi(x_i);w^{(\widetilde{d}_i)})} 
     \Big),
\end{equation}
where $\widetilde{d}_i$ are drawn i.i.d. from the marginal distribution $\hat{\pr}_D$. Summarizing, the overall objective is
\begin{equation}\label{equ:feature_update}
      \min_{\theta,\Phi} L_G(\Phi, \theta) + \lambda L_R(\Phi).
\end{equation}
As shown in Algorithm \ref{alg:suff}, we train our model by alternating between the fitting subgroup-specific models in \eqref{equ:fitting} and feature updating in \eqref{equ:feature_update}. 
\subsection{Imposing calibration for mean and variance}
\label{subsec:residue}
To achieve the calibration for mean and variance, we let the representation $\Phi = (\Phi_1, \Phi_2)$. Then, to enable the use of the residual-based neural network \cite{hall1989variance}, we let the conditional expectation depend only on $\Phi_1$ i.e., $f(\Phi(X)) = f(\Phi_1(X))$ and let the conditional variance depend only on $\Phi_2$ i.e., $g(\Phi(X)) = g(\Phi_2(X))$. This method is useful in scenarios where the conditional Gaussian assumption in the Section \ref{subsec:hetero} does not hold.

In a residual-based neural network, the conditional mean-prediction network is trained by minimizing:
$$  L_{S1} (\Phi_1, \theta_f) \triangleq \sum_{i=1}^n (y_i - f (\Phi_1(x_i); \theta_f) )^2.$$
The feature representation $\Phi_1$ is parameterized by $\theta_{\Phi_1}$ and the mean-prediction network is supposed to learn the parameters $\theta_{\Phi_1}$ and $\theta_f$.

Then, the conditional variance-prediction network is trained by fitting the residuals obtained from the mean-prediction network, i.e., $r_i \triangleq (y_i - f (\Phi_1(x_i); \theta_f))^2$ by minimizing:
$$L_{S2} (\Phi_2, \theta_g) \triangleq  \sum_{i=1}^n (r_i - g (\Phi_2(x_i); \theta_g) )^2.$$
The feature representation $\Phi_2$ is parameterized by $\theta_{\Phi_2}$ and the variance-prediction network is supposed to learn the parameters $\theta_{\Phi_2}$ and $\theta_g$.

To impose calibration under mean, we need to convert the following hard constraint
\begin{align}
    \ev[Y | \Phi_1(X), D] = \ev[Y | \Phi_1(X)] \label{eq:cali_mean_1}
\end{align}
into a soft constraint. We do this by using the following contrastive loss: 
\begin{align}\label{equ:MSE_contrastive}
    &\Ed{D}{ \Ed{\Phi(X),Y}{(Y-\ev[Y|\Phi_1(X),D])^2}}  \nonumber \\
     &\qquad- \Ed{\Phi(X),Y,D}{(Y-\ev[Y|\Phi_1(X),D])^2},
\end{align}
which is inspired from \eqref{bound} and obtained by replacing the negative log-likelihood $-\log \pr (Y|\Phi(X),D)$ in \eqref{bound} by the MSE achieved using the representation $\Phi_1(X)$ and sensitive attribute $D$. We emphasize that \eqref{equ:MSE_contrastive} is zero when \eqref{eq:cali_mean_1} holds and therefore \eqref{equ:MSE_contrastive} is a relaxation of \eqref{eq:cali_mean_1}.

To compute \eqref{equ:MSE_contrastive}, we need to learn the unknown conditional expectation $\ev[Y|\Phi_1(X),D]$. Similar to Section \ref{subsec:hetero}, we approximate this by $f(y|\Phi_1(x);w^{(d)}_f)$ (i.e., train a \textit{subgroup-specific} mean-prediction model with parameters $w^{(d)}_f$ for each $d\in \cD$, instead of using $D$ as a separate input).
Similar to Section \ref{subsec:hetero}, combining everything and replacing all the expectations in \eqref{equ:MSE_contrastive} with empirical averages, the  regularizer for the mean-prediction network is given by
\begin{align}\label{equ:RS1}
    L_{R1}(\Phi_1) \triangleq  \sum_{i=1}^n &\Big(\big(y_i - f (\Phi_1(x_i); w_f^{(\widetilde{d}_i)}) \big)^2 \\
    &\qquad - \big(y_i - f (\Phi_1(x_i); w_f^{(d_i)}) \big)^2\Big), \nonumber
    \vspace{-2mm}
\end{align}
where $\widetilde{d}_i$ are drawn i.i.d. from ${\pr}_D$ i.e., the marginal distribution of $D$ (approximated by $\hat{\pr}_D$) and for $d \in \cD$,
$$w_f^{(d)} = \argmin_{w}  \sum_{i:\ d_i=d} \big(y_i - f (\Phi_1(x_i); w) \big)^2.$$
In summary, the overall objective for mean-prediction is
\begin{align}
    \min_{\theta_f,\Phi_1} L_{S1} (\Phi_1,\theta_f) + \lambda_1 L_{R1}(\Phi_1). \label{eq:in1}
\end{align}
Once the mean-prediction network is trained, we obtain the residuals and train the variance-prediction network $g$ using a similar regularizer:
\begin{align}
    \min_{\theta_g,\Phi_2} L_{S2} (\Phi_2,\theta_g) + \lambda_2 L_{R2}(\Phi_2), \label{eq:in2}
\end{align}
where $L_{S2}$ and $L_{R2}$ are defined in a similar manner as $L_{S1}$ and $L_{R1}$. More details about these and the pseudo-code (i.e., Algorithm \ref{alg:two_stage}) are provided in the Appendix~\ref{appendix:algo_details}.

%% file: content/5experiments.tex
\section{Experimental Results}
\label{sec:expts}
\textbf{Datasets.} 
We test our 
algorithms on 
Insurance and Crime datasets, and provide an application of our method in Causal Inference via IHDP  dataset.
These datasets (summarized in Table \ref{table:datasets}) are selected due to their potential fairness concerns, e.g., (a) presence of features often associated with possible discrimination, such as race and sex, and (b) potential sensitivity regarding the predictions being made such as medical expenses, violent crimes, and cognitive test score.
\begin{table}[h]
\caption{Summary of datasets.}
\label{table:datasets}
\centering
\begin{tabular}{p{18mm}p{40mm}p{14mm}}\\
\toprule
\textbf{Dataset} & \textbf{Target} & \textbf{Attribute} \\
\midrule
Insurance & Medical Expenses & Sex \\
\midrule
Crime & Crimes per Population & Race \\
\midrule
IHDP & Cognitive Test Score & Sex \\
\bottomrule
\end{tabular}
\end{table}

\noindent \textbf{Insurance.} The Insurance dataset \citep{lantz2019machine} 
considers predicting total annual medical expenses charged to patients using demographic statistics.
Following \cite{chi2021understanding}, we use sex as the sensitive attribute: $D=1$ (i.e., minority) if male otherwise 0.
After preprocessing (see Appendix \ref{appendix:experiments}), the dataset contains 1000 samples (338 with  $D= 1$ and 662 with $D= 0$) and 5 features. 

\noindent \textbf{Communities and Crime.} The Crime dataset \citep{redmond2002data} considers predicting the number of violent crimes per 100K population using socio-economic information of communities in U.S.
Following \cite{chi2021understanding}, we use race (binary)\footnote[4]{We provide results for the scenario where race can take more than two values in in Appendix \ref{subsec:three_subgroups}.} as the sensitive attribute: $D\!=\!1$ (i.e., minority) if the population percentage of the black is more or equal to 20 otherwise 0. 
After preprocessing (see Appendix \ref{appendix:experiments}), the dataset contains 1994 samples (532 with $D\!=\! 1$ and 1462 with $D\!=\!0$) and 99 features.\\

\noindent \textbf{IHDP.} The IHDP dataset \citep{hill2011bayesian} is generated based on a randomized control trial targeting low-birth-weight, premature infants. In the treated group, the infants were provided with both intensive and high-quality childcare and specialist home visits. 
The task is to predict the infants' cognitive test scores.
We let sex be the sensitive attribute and observe that \textit{male} is the minority group (i.e., $D = 1$) in both the control group and the treatment group. 
After preprocessing (see Appendix \ref{appendix:experiments}), the control group has 608 samples (296 with $D = 1$ and 312 with $D = 0$) and the treatment group has 139 samples (67 with $D = 1$ and 72 with $D = 0$). The dataset contains 25 features.\\

\noindent \textbf{Choice of $\lambda$.} We observe our algorithms to be agnostic to the choice of $\lambda$ as long as it is in a reasonable range, i.e., $\lambda \in [0.5, 3]$. To be consistent, we set $\lambda = 1$ throughout. \\

\noindent \textbf{Baselines.} We compare against the following baselines:
\begin{itemize}[leftmargin=*,topsep=-2pt,itemsep=-4pt]
    \item \texttt{Baseline 1}: Heteroskedastic neural network without any regularizer i.e., Algorithm \ref{alg:suff} with $\lambda = 0$.
    \item \texttt{Baseline 2}: Residual-based neural network without any regularizer i.e., Algorithm \ref{alg:two_stage} with $\lambda = 0$.\\
\end{itemize}

\noindent \textbf{Experimental Details.}
In all of our experiments, we use two-layer neural networks, and train our model only once on a fixed training set. We evaluate and report the empirical findings on a held-out test set with a train-test split ratio of $0.8/0.2$. More experimental details can be found in Appendix \ref{appendix:experiments}.\\

\noindent \textbf{Comparison in terms of selective regression.}
To compare different algorithms in terms of how well they perform selective regression (i.e., without fairness), we look at area under MSE vs. coverage curve (AUC), which encapsulates performance across different coverage \cite{franc2019discriminative, lee2021fair}. We provide the results in Table \ref{table:aucs_app} (smaller AUC indicates better performance)
and observe that our algorithms are competitive (if not better) than baselines. We provide MSE vs. coverage curves in Appendix \ref{appendix:experiments}.\\

\noindent \textbf{Comparison in terms of fairness.}
To compare different algorithms in terms of  \emph{fair} selective regression, we look at subgroup MSE vs coverage curves. for the Insurance dataset, we show these curves for \texttt{Baseline 2} in Figure \ref{fig:insurance_baseline_two} and Algorithm \ref{alg:two_stage} in Figure \ref{fig:insurance_suff_two}. for the Crime dataset, we show these curves for \texttt{Baseline 1} in Figure \ref{fig:crime_baseline_one} and Algorithm \ref{alg:suff} in Figure \ref{fig:crime_suff_one}. 
See Appendix \ref{appendix:experiments} for remaining set of curves.

\begin{figure*}[h]
\centering
\begin{subfigure}{.32\textwidth}
\centering
  \captionsetup{justification=centering}
  \includegraphics[width=0.85\linewidth]{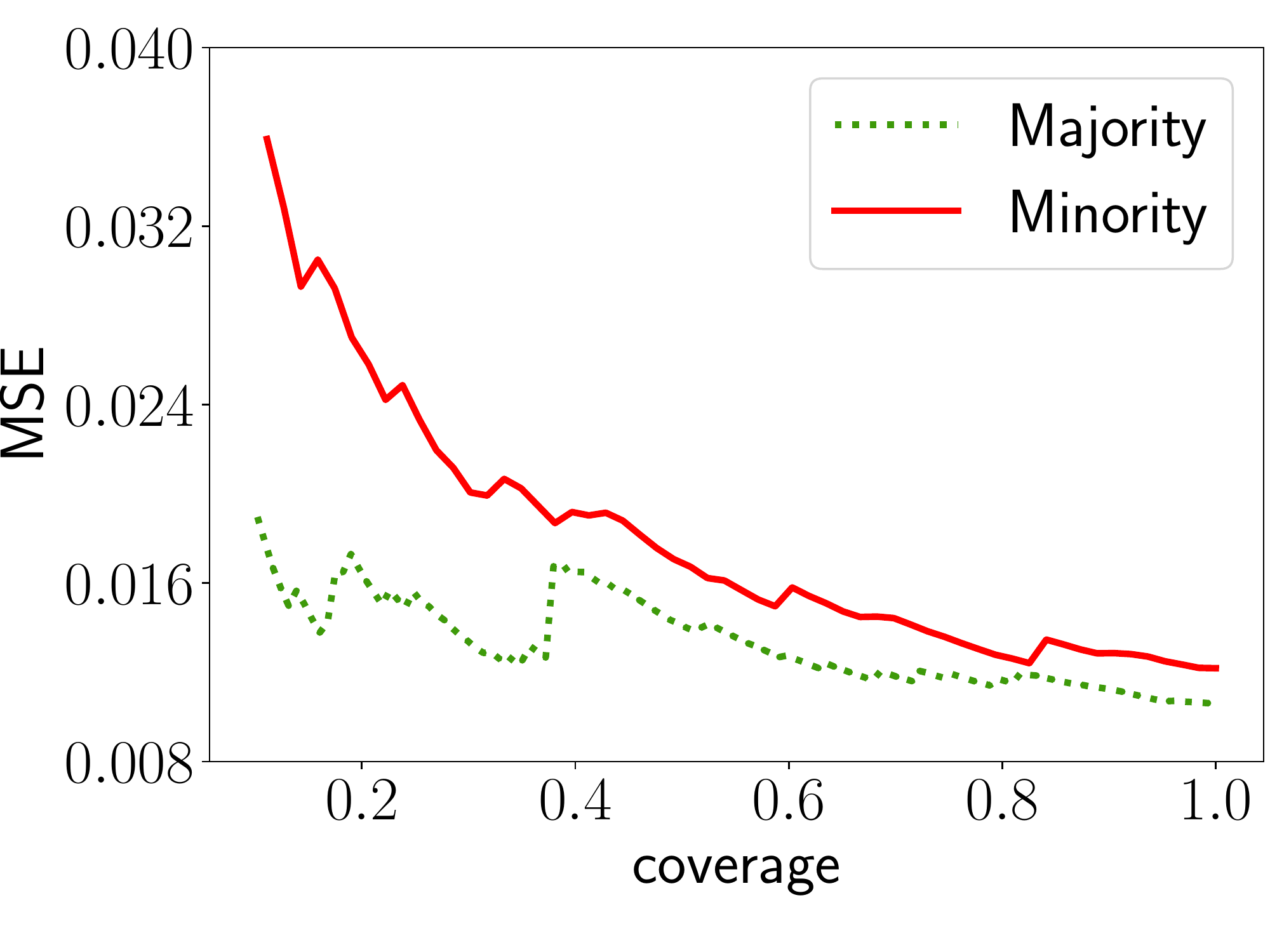}
  \caption{Performance of \texttt{Baseline 2}\\ for the Insurance dataset.}
  \label{fig:insurance_baseline_two}
  \includegraphics[width=0.85\linewidth]{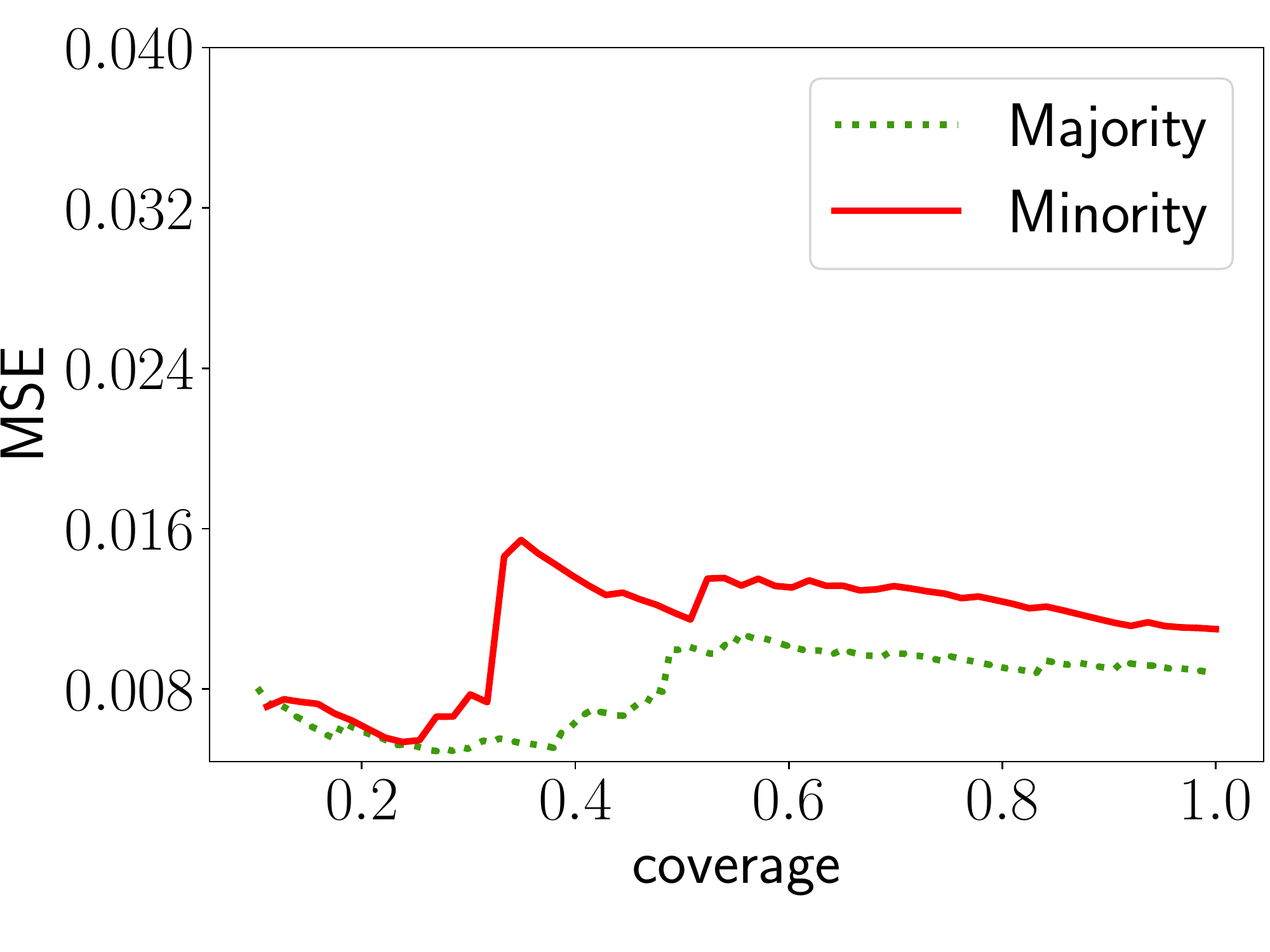}
 \caption{Performance of Algorithm \ref{alg:two_stage} \\for the Insurance dataset.}
  \label{fig:insurance_suff_two}
\end{subfigure}\hfill
\begin{subfigure}{.32\textwidth}
\centering
  \captionsetup{justification=centering}
  \includegraphics[width=0.85\linewidth]{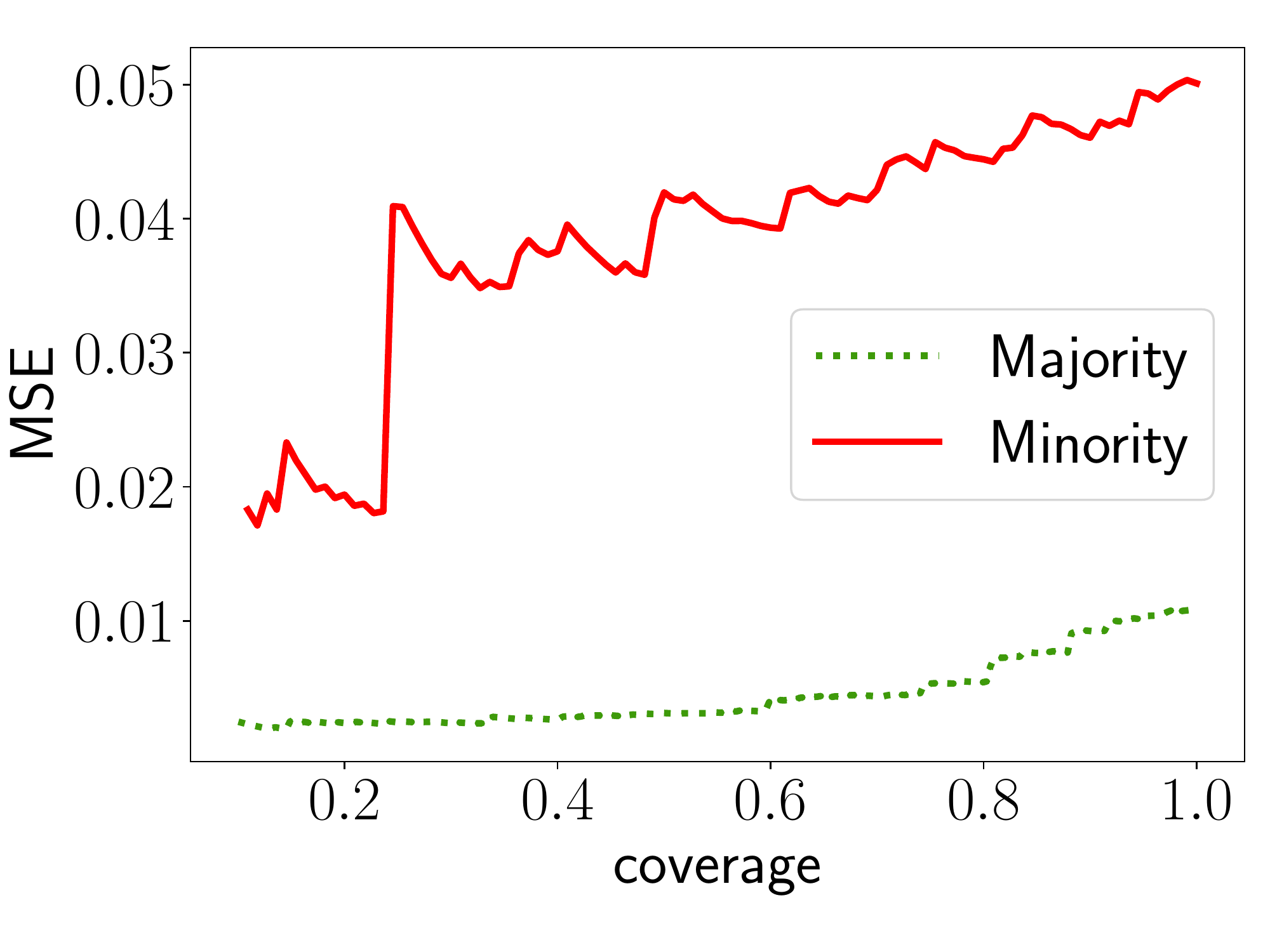}
  \caption{Performance of \texttt{Baseline 1} \\ for the Crime dataset.}
  \label{fig:crime_baseline_one}
  \includegraphics[width=0.85\linewidth]{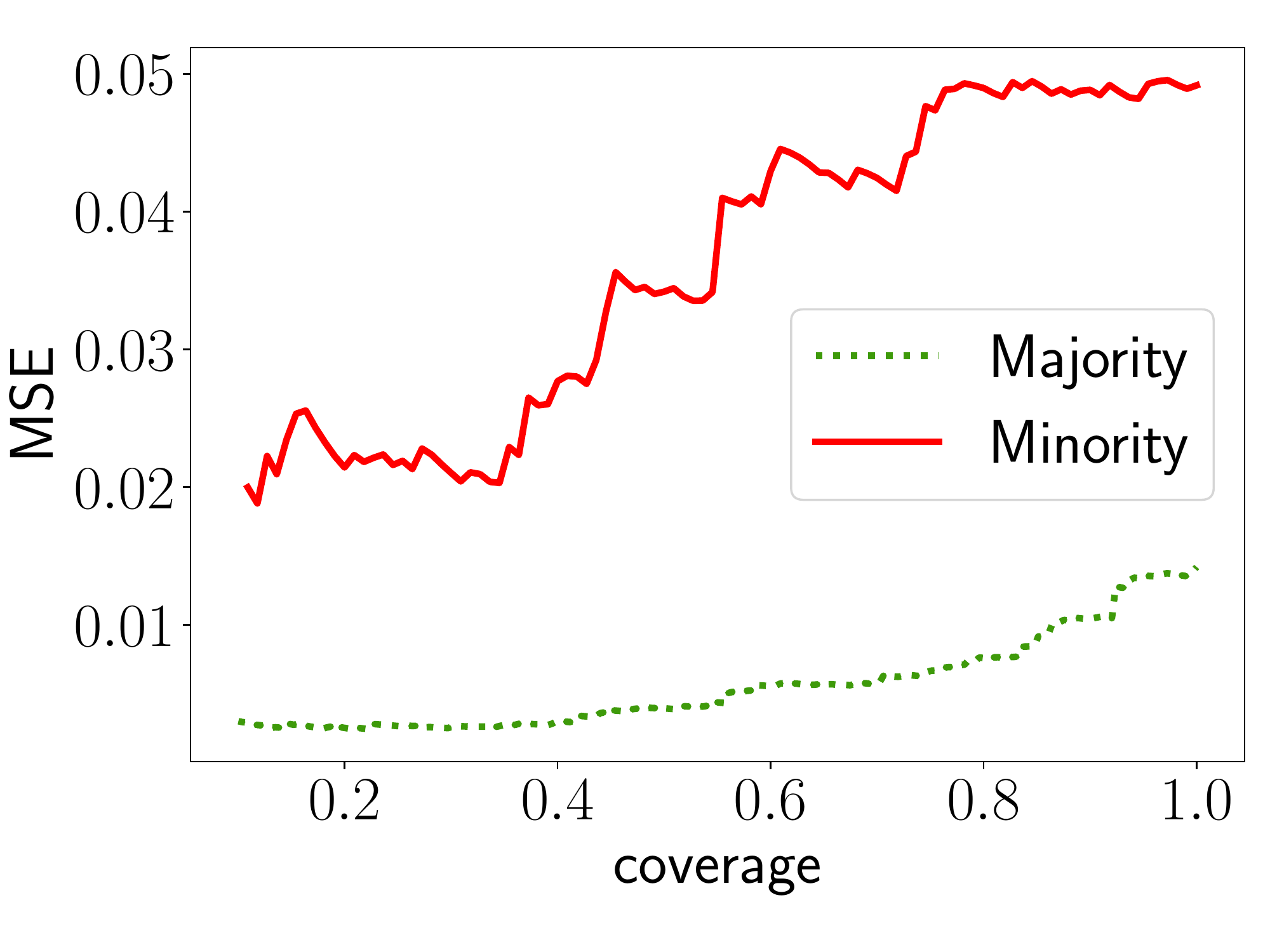}
  \caption{Performance of Algorithm \ref{alg:suff} \\ for the Crime dataset.}
  \label{fig:crime_suff_one}
\end{subfigure}\hfill
\begin{subfigure}{.32\textwidth}
\centering
  \captionsetup{justification=centering}
  \includegraphics[width=0.83\linewidth]{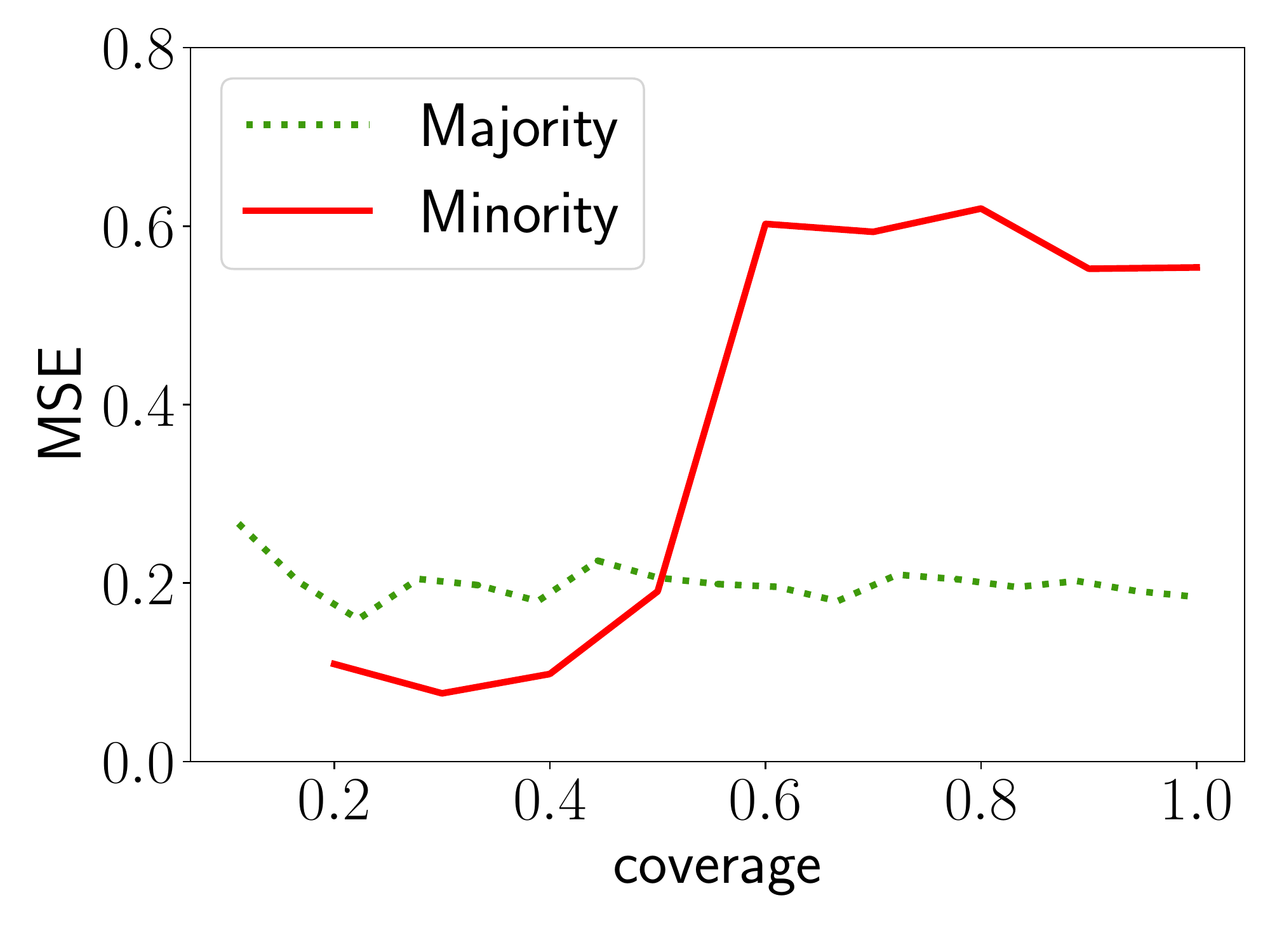}
  \caption{Performance of \texttt{Baseline 1} \\ for IHDP (treatment) dataset.}
  \label{fig:ihdp_treatment_baseline}
  \includegraphics[width=0.83\linewidth]{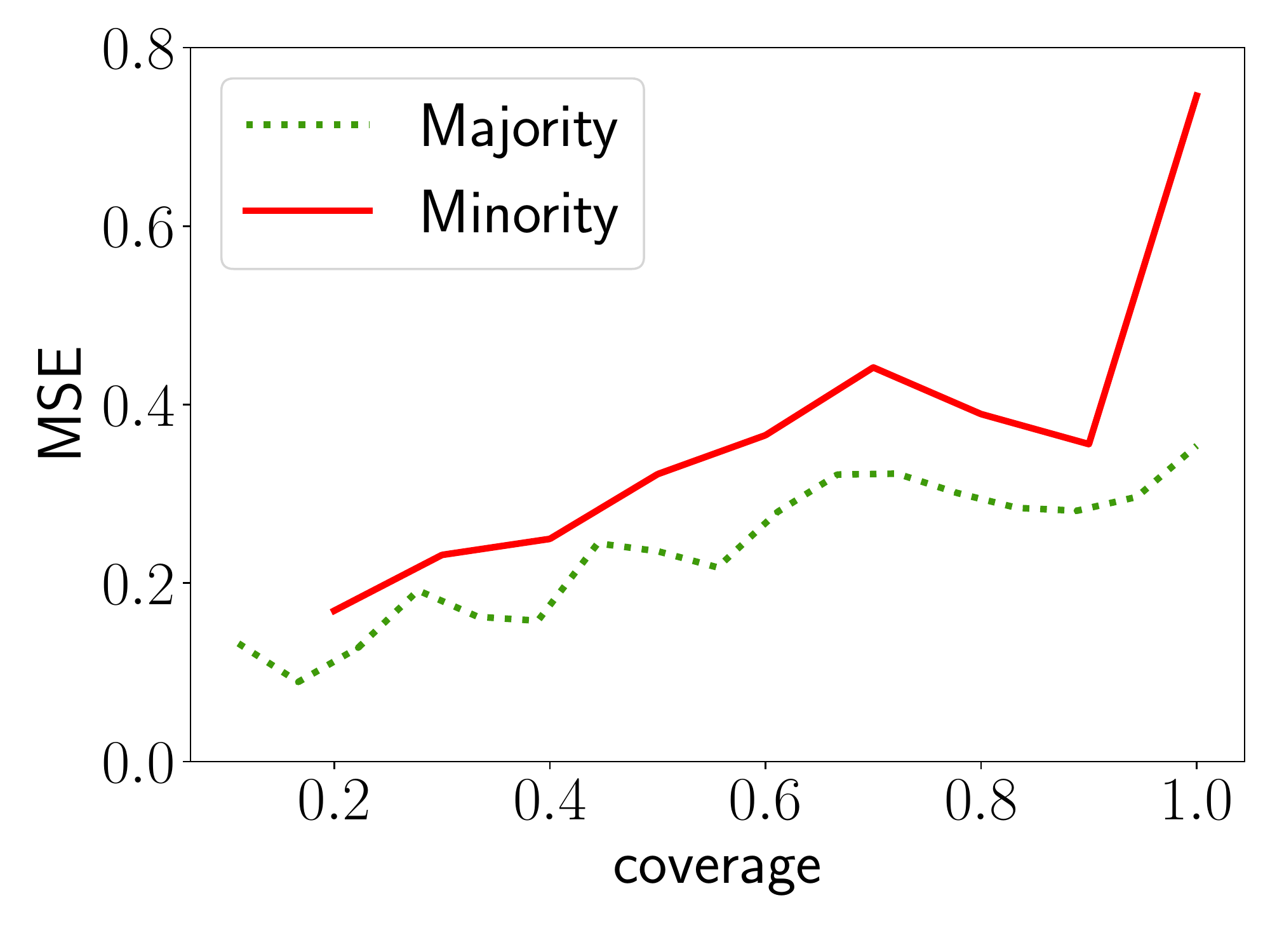}
 \caption{Performance of Algorithm \ref{alg:suff} \\ for IHDP (treatment) dataset.}
   \label{fig:ihdp_treatment_sufficiency}
\end{subfigure}
\caption{Subgroup MSE vs. coverage plots. Compared to baselines (top), our algorithms (bottom) (a) show a consistent trend of decreasing MSE with decrease in coverage for both subgroups, (b) achieve better minority MSE for fixed coverage, i.e., a smaller AUC for minority subgroup (red),
(c) achieve comparable majority MSE for fixed coverage, i.e., a comparable AUC for majority subgroup (green),
(d) reduce gap between the subgroup MSE curves, i.e., a smaller AUADC.}
\label{fig:main}
\vspace{-0.5em}
\end{figure*}
For the Insurance dataset, we see that subgroup MSE for \texttt{Baseline 2} increases with decrease in coverage for both majority and minority subgroups (Figure \ref{fig:insurance_baseline_two}). In contrast, the subgroup MSE for Algorithm \ref{alg:two_stage} tends to decrease with a decrease in coverage for both subgroups (Figure \ref{fig:insurance_suff_two}). For the Crime dataset, we see that the subgroup MSE for \texttt{Baseline 1} as well as Algorithm \ref{alg:suff} tends to decrease with a decrease in coverage for both subgroups (Figure \ref{fig:crime_baseline_one} and Figure \ref{fig:crime_suff_one}). However, for a particular coverage, Algorithm \ref{alg:suff} achieves a better MSE for the minority subgroup, a comparable MSE for the majority subgroup, and reduces the gap between the subgroup curves than \texttt{Baseline 1}. 

As hinted above, in addition to ensuring \textit{monotonic selective risk}, one may wish to (a) achieve a better performance focusing solely on the majority subgroup, (b) achieve a better performance focusing solely on the minority subgroup, and (c) reduce the gap between the minority subgroup MSE and the majority subgroup MSE across all thresholds.  These aspects could be quantitatively captured by looking at (a) the area under the majority MSE vs. coverage curve, i.e., AUC (D = 0), (b) the area under the minority MSE vs. coverage curve, i.e., AUC (D = 1), and (c) 
the area under the absolute difference of the subgroup MSE vs coverage curves (AUADC) \cite{franc2019discriminative,lee2021fair} respectively.
We provide these results in Table \ref{table:aucs_app} 
and observe that our algorithms outperform the baselines across datasets in terms of AUC (D = 1) and AUADC while being comparable in terms of AUC (D = 0).\\

\begin{table}[h]
\caption{Mean $\pm$ standard deviation (averaged across 5 runs) for AUC, AUC (D = 0), AUC (D = 1), and AUADC for all the algorithms and all the datasets. Smaller values are better.}
\label{table:aucs_app}
\centering
\begin{tabular}{p{16mm}p{20mm}p{26mm}p{26mm}p{26mm}p{26mm}}\\
\toprule
\textbf{Dataset} & \textbf{Algorithm} & \textbf{AUC}  & \textbf{AUC} (D = 0) & \textbf{AUC} (D = 1) & \textbf{AUADC} \\
\midrule
Insurance & \texttt{Baseline 1} & 0.0371 $\pm$ 0.0255 & 0.0342 $\pm$ 0.0197 & 0.0442 $\pm$ 0.0218 & 0.0069 $\pm$ 0.0050\\
 &  Algorithm \ref{alg:suff} & 0.0195 $\pm$ 0.0059 & 0.0207 $\pm$ 0.0050 & 0.0167 $\pm$ 0.0075 & 0.0052 $\pm$ 0.0031\\
 & \texttt{Baseline 2} & 0.0142 $\pm$ 0.0052 & 0.0129 $\pm$ 0.0042 & 0.0175 $\pm$ 0.0026 & 0.0079 $\pm$ 0.0041 \\
 &  Algorithm \ref{alg:two_stage} & 0.0099 $\pm$ 0.0006  & 0.0087 $\pm$ 0.0004 & 0.0120 $\pm$ 0.0011 & 0.0051 $\pm$ 0.0018 \\
\midrule
Crime &  \texttt{Baseline 1} & 0.0075 $\pm$ 0.0002 & 0.0040 $\pm$ 0.0011 & 0.0345 $\pm$ 0.0037 & 0.0309 $\pm$ 0.0008 \\
&  Algorithm \ref{alg:suff} & 0.0079 $\pm$ 0.0004 & 0.0045 $\pm$ 0.0013 & 0.0296 $\pm$ 0.0054 & 0.0298 $\pm$ 0.0011 \\
&  \texttt{Baseline 2} & 0.0101 $\pm$ 0.0019  & 0.0060 $\pm$ 0.0017 & 0.0442 $\pm$ 0.0022  & 0.0272 $\pm$ 0.0013 \\
&  Algorithm \ref{alg:two_stage} & 0.0117 $\pm$ 0.0017  & 0.0082 $\pm$ 0.0012 & 0.0375 $\pm$ 0.0019 & 0.0257 $\pm$ 0.0028 \\
\midrule
IHDP  & \texttt{Baseline 1} & 0.3053 $\pm$ 0.0823 & 0.2000 $\pm$ 0.0899 & 0.3509 $\pm$ 0.0811 & 0.2266 $\pm$ 0.0919 \\
(Treatment) &  Algorithm \ref{alg:suff} & 0.2435 $\pm$ 0.0823 & 0.2024 $\pm$ 0.0935 & 0.2849 $\pm$ 0.0767  & 0.2034 $\pm$ 0.0925\\
\midrule
IHDP & \texttt{Baseline 1} & 0.2041 $\pm$ 0.0138 & 0.2144 $\pm$ 0.0101 & 0.1983 $\pm$ 0.0125 & 0.0495 $\pm$ 0.0053 \\
(Control) &  Algorithm \ref{alg:suff} & 0.2017 $\pm$ 0.0170 & 0.2169 $\pm$ 0.0133 & 0.1877 $\pm$ 0.0129 & 0.0398 $\pm$ 0.0073 \\
\bottomrule
\end{tabular}
\end{table}

\noindent \textbf{Application to Causal Inference.}  We provide an application of our work to fair-treatment effect estimation. Treatment effect estimation has been studied under the paradigm of \textit{prediction with reject option} \cite{jesson2020identifying}, and we explore an additional dimension to this, i.e., the behavior across sub-populations. 
We follow the standard approach \citep{Kunzel2019} of viewing the dataset as two distinct datasets --- one corresponding to treatment group and the other corresponding to control group --- and apply our framework to these distinct datasets.
We focus only on Algorithm \ref{alg:suff} \cite{hill2011bayesian} and compare with \texttt{Baseline 1} since the simulated target for the IHDP dataset perfectly fits the conditional Gaussian assumption \cite{hill2011bayesian}.
For treatment group, we provide the subgroup MSE vs coverage behavior for \texttt{Baseline 1} in Figure \ref{fig:ihdp_treatment_baseline} and Algorithm \ref{alg:suff} in Figure \ref{fig:ihdp_treatment_sufficiency}. 
We provide corresponding curves for the control group in Appendix \ref{appendix:experiments}.
Compared to \texttt{Baseline 1} (Figure \ref{fig:ihdp_treatment_baseline}), Algorithm \ref{alg:suff} (Figure \ref{fig:ihdp_treatment_sufficiency}) (a) shows a general trend of decreasing MSE with decrease in coverage for both subgroups and (b) achieves a better minority MSE for a fixed coverage. 
Table \ref{table:aucs_app} suggests that our algorithm outperforms the baseline in terms of AUC (D = 1) and AUADC while being comparable in AUC (D = 0).

%% file: content/6conclusion.tex
\section{Concluding Remarks}
We proposed a new fairness criterion, \textit{monotonic selective risk} for selective regression, which requires the performance of each subgroup to improve with a decrease in coverage. We provided two conditions for the feature representation (sufficiency and calibrated for mean and variance) under which the proposed fairness criterion is met. We presented algorithms to enforce these conditions and demonstrated mitigation of disparity in the performances across groups for three real-world datasets. 

Monotonic selective risk is one criteria for fairness in selective regression. 
Developing and investigating other such criteria and understanding their relationship with monotonic selective risk is an important question for future research. 

%% file: content/7appendix.tex
\textbf{Organization.}
The Appendix is organized as follows. In Appendix \ref{appendix:proof_theorem}, we provide the proof of Theorem \ref{thm:sufficiency}.
In Appendix \ref{appendix:algo_details}, we provide more details for imposing calibration for mean and variance via the residual-based neural network as well as provide the pseudo-code, i.e., Algorithm \ref{alg:two_stage}. 
In Appendix \ref{appendix:experiments}, we provide more experimental details and results.
\section{Proof of Theorem \ref{thm:sufficiency}}
\label{appendix:proof_theorem}
We restate the Theorem below and then provide a proof.
\thmSufficiency*
\begin{proof}
First, let us simplify the expression for $\MSE(f, g, \tau, d)$. We have
\begin{align}
    \MSE(f, g, \tau, d) & = \frac{\ev[(Y-f(\Phi(X)))^2 \cdot \Indicator(g(\Phi(X)) \le \tau) | D=d]}{\ev[\Indicator ( g(\Phi(X)) \le \tau) |  D=d]} \\
    & = \frac{\ev[(Y-f(\Phi(X)))^2 \cdot \Indicator(g(\Phi(X)) \le \tau) | D=d]}{\pr(g(\Phi(X)) \le \tau |  D=d)} \label{eq:mse_simple0}
\end{align}
Let $c_{\tau, d} \triangleq \pr(g(\Phi(X)) \le \tau |  D=d)$. Using \eqref{eq:mse_simple0}, we have
\begin{align}
    \MSE(f, g, \tau, d) - \MSE(f, g, \tau', d) &  = \frac{1}{c_{\tau, d}} \ev[(Y-f(\Phi(X)))^2 \cdot \Indicator(g(\Phi(X)) \le \tau) | D=d] \\
    & - \frac{1}{c_{\tau', d}} \ev[(Y-f(\Phi(X)))^2 \cdot \Indicator(g(\Phi(X)) \le \tau') | D=d] \\
    & = \bigg(\frac{1}{c_{\tau, d}} - \frac{1}{c_{\tau', d}}\bigg) \ev[(Y-f(\Phi(X)))^2 \cdot \Indicator(g(\Phi(X)) \le \tau) | D=d] \label{eq:mse_simple1}\\
    & - \frac{1}{c_{\tau', d}} \ev[(Y-f(\Phi(X)))^2 \cdot \Indicator(\tau < g(\Phi(X)) \le \tau') | D=d] \label{eq:mse_simple2}
\end{align}
Now, let us upper bound $\ev[(Y-f(\Phi(X)))^2 \cdot \Indicator(g(\Phi(X)) \le \tau) | D=d]$. We have
\begin{align}
    & \ev[(Y-f(\Phi(X)))^2 \cdot \Indicator(g(\Phi(X)) \le \tau) | D=d] \label{eq:proof1}\\
    & \qquad \stackrel{(a)}{=} \ev_{\Phi(X)|D}[\Indicator(g(\Phi(X)) \le \tau) \cdot \ev_{Y|\Phi(X), D}[(Y-f(\Phi(X)))^2 | \Phi(X), D=d] | D = d] \label{eq:proof2}\\
    & \qquad \stackrel{(b)}{=} \ev_{\Phi(X)|D}[\Indicator(g(\Phi(X)) \le \tau) \cdot \ev_{Y|\Phi(X), D}[(Y-\ev[Y | \Phi(X)])^2 | \Phi(X), D=d] | D = d] \label{eq:proof3}\\
    & \qquad \stackrel{(c)}{=} \ev_{\Phi(X)|D}[\Indicator(g(\Phi(X)) \le \tau) \cdot \ev_{Y|\Phi(X), D}[(Y-\ev[Y | \Phi(X), D = d])^2 | \Phi(X), D=d] | D = d] \label{eq:proof4}\\
    & \qquad \stackrel{(d)}{=} \ev_{\Phi(X)|D}[\Indicator(g(\Phi(X)) \le \tau) \cdot \Var[Y | \Phi(X), D = d] | D = d ] \label{eq:proof5}\\
    & \qquad \stackrel{(e)}{=} \ev_{\Phi(X)|D}[\Indicator(g(\Phi(X)) \le \tau) \cdot \Var[Y | \Phi(X)] | D = d ] \label{eq:proof6}\\
    & \qquad \stackrel{(f)}{=} \ev_{\Phi(X)|D}[\Indicator(g(\Phi(X)) \le \tau) \cdot g(\Phi(X)) | D = d] \label{eq:proof7}\\
    & \qquad \leq \tau \ev_{\Phi(X)|D}[\Indicator(g(\Phi(X)) \le  \tau) | D = d] 
    =  \tau \pr(g(\Phi(X)) \le \tau |  D=d) 
    = \tau c_{\tau, d} \label{eq:proof9}
\end{align}
where $(a)$ follows from the definition of conditional expectation and because $\Indicator(g(\Phi(X)) \le \tau)$ is a constant conditioned on $\Phi(X)$, $(b)$ follows because $f(\Phi(X)) = \ev[Y | \Phi(X)]$, $(c)$ follows because $Y\!\perp\!D | \Phi(X) \implies \ev[Y | \Phi(X)] = \ev[Y | \Phi(X), D = d]$, $(d)$ follows from the definition of conditonal variance, $(e)$ follows because $Y\!\perp\!D | \Phi(X) \implies \Var[Y | \Phi(X)] = \Var[Y | \Phi(X), D = d]$, and $(f)$ follows because $g(\Phi(X)) = \Var[Y | \Phi(X)]$.

Now, let us lower bound $\ev[(Y-f(\Phi(X)))^2 \cdot \Indicator(\tau < g(\Phi(X)) \le \tau') | D=d]$. 
Similar to above, we have
\begin{align}
    & \ev[(Y-f(\Phi(X)))^2 \cdot \Indicator(\tau < g(\Phi(X)) \le \tau') | D=d] \\
    & = \ev_{\Phi(X)|D}[\Indicator(\tau < g(\Phi(X)) \le \tau') \cdot g(\Phi(X)) | D = d] \label{eq:proof10}\\
    & > \tau \ev_{\Phi(X)|D}[\Indicator(\tau < g(\Phi(X)) \le  \tau') | D = d] \label{eq:proof11}\\
    & = \tau (\pr(g(\Phi(X)) \le \tau' |  D=d) - \pr(g(\Phi(X)) \le \tau |  D=d)) \label{eq:proof12}\\
    & = \tau (c_{\tau', d} - c_{\tau, d}) \label{eq:proof13}
\end{align}
Plugging in \eqref{eq:proof9} and \eqref{eq:proof13} in \eqref{eq:mse_simple1} and \eqref{eq:mse_simple2}, we have
\begin{align}
    \MSE(f, g, \tau, d) - \MSE(f, g, \tau', d) < \bigg(\frac{1}{c_{\tau, d}} - \frac{1}{c_{\tau', d}}\bigg) \cdot \tau c_{\tau, d} - \frac{1}{c_{\tau', d}} \tau \cdot (c_{\tau', d} - c_{\tau, d}) = 0.
\end{align}
\end{proof}
\section{More Details to Impose Calibration for Mean and Variance}
\label{appendix:algo_details}
In this section, we provide more details for imposing calibration for mean and variance via the residual-based neural network as well as provide the pseudo-code. As described in Section~\ref{subsec:residue}, in a residual-based neural network, once the mean-prediction network $f$ is trained, the residuals i.e., $r_i = (y_i - f (\Phi_1(x_i); \theta_f))^2$ are used to train the variance-prediction network $g$ by minimizing:
\begin{equation}
  L_{S2} (\Phi_2,\theta_g) \triangleq \sum_{i=1}^n (r_i - g (\Phi_2(x_i); \theta_g) )^2.
\end{equation}
The feature representation $\Phi_2$ is parameterized by $\theta_{\Phi_2}$ and the variance-prediction network is supposed to learn the parameters $\theta_{\Phi_2}$ and $\theta_g$.

To impose calibration under variance, we construct a contrastive loss similar to \eqref{equ:MSE_contrastive}. Then, the regularizer can be written as
\begin{equation}\label{equ:RS2}
    L_{R2}( \Phi_2) \triangleq \sum_{i=1}^n\Big(\big(r_i - g (\Phi_2(x_i); w_g^{(\widetilde{d}_i)}) \big)^2   - \big(r_i - g (\Phi_2(x_i); w_g^{({d}_i)}) \big)^2\Big),
\end{equation}
where  $\widetilde{d}_i$ are drawn i.i.d. from the marginal distribution $P_D$,
and for $d \in \cD$,
\begin{equation}
    w_g^{(d)} = \argmin_{w}  \sum_{i:\ d_i=d} \big(r_i - g (\Phi_2(x_i); w) \big)^2. 
\end{equation}
Summarizing, the overall objective for variance-prediction is
\begin{equation}
    \min_{\theta_g,\Phi_2} L_{S2} (\Phi_2,\theta_g) + \lambda_2 L_{R2}(\Phi_2).
\end{equation}
We provide a pseudo-code in Algorithm \ref{alg:two_stage} where 
\begin{align}
    L_{d1}(w_f) & \triangleq \sum_{i:\ d_i=d} \big(y_i - f (\Phi_1(x_i); w_f) \big)^2, \\
    L_{d2}(w_g) & \triangleq \sum_{i:\ d_i=d} \big(r_i - g (\Phi_2(x_i); w_g) \big)^2. 
\end{align}
\begin{algorithm}[tb]
   \caption{Residual-based neural network with calibration-based regularizer}
\label{alg:two_stage}
\begin{algorithmic}
\STATE {\bfseries Input:} training samples $\{(x_i,y_i,d_i)\}_{i=1}^{n}$, regularizers $\lambda_1$ and $\lambda_2$
 \STATE {\bfseries Draw:} $\{\widetilde{d}_1,\dots,\widetilde{d}_n\}$ drawn i.i.d. from $\hat{\pr}_D$
 \STATE {\bfseries Initialize:} $\theta_f$, $\theta_g$, $\theta_{\Phi_1}$, $\theta_{\Phi_2}$, $w_f^{(d)}$, and $w_g^{(d)}$ with pre-trained models 
\STATE {\bfseries Initialize:} $n_d = $ number of samples in group $d$ $\forall d \in \cD$
\FOR{each training iteration}
    \FOR{each batch}
    \FOR{$d=1,\dots,|\cD|$} 
        \STATE {
        $w_f^{(d)} \leftarrow w_f^{(d)} - \frac{1}{n_d} \eta_{f} \nabla_{w_f} L_{d1}(w_f)$ \textBlue{\# update subgroup-specific mean predictor}}
        \ENDFOR
    \ENDFOR
    \FOR{each batch}
    \STATE {$\theta_{\Phi_1} \leftarrow \theta_{\Phi_1} - \frac{1}{n} \eta \nabla_{\theta_{\Phi_1}} (L_{S1}(\Phi_1, \theta_f)+\lambda_1 L_{R1}(\Phi_1))$ \textBlue{\# update feature extractor for mean predictor}}
    \STATE {$\theta_f \leftarrow \theta_f - \frac{1}{n} \eta \nabla_{\theta_f} L_{S1}(\Phi_1, \theta_f)$ \textBlue{\# update mean predictor}}
    \ENDFOR
\ENDFOR    
    \STATE Compute the residuals: $r_i = (y_i - f (\Phi_1(x_i); \theta_f))^2$ 
\FOR{each training iteration}
    \FOR{each batch}
    \FOR{$d=1,\dots,|\cD|$}
        \STATE {
        $w_g^{(d)} \leftarrow w_g^{(d)} - \frac{1}{n_d} \eta_{g} \nabla_{w_g} L_{d2}(w_g)$ \textBlue{\# update subgroup-specific variance predictor}}
        \ENDFOR
    \ENDFOR
    \FOR{each batch}
    \STATE {$\theta_{\Phi_2} \leftarrow \theta_{\Phi_2} - \frac{1}{n} \eta \nabla_{\theta_{\Phi_2}} (L_{S2}(\Phi_2, \theta_g)+\lambda_2 L_{R2}(\Phi_2))$ \textBlue{\# update feature extractor for variance predictor}} 
    \STATE {$\theta_g \leftarrow \theta_g - \frac{1}{n} \eta \nabla_{\theta_g} L_{S2}(\Phi_2, \theta_g)$ \textBlue{\# update variance predictor}} 
    \ENDFOR
\ENDFOR 
\end{algorithmic}
\end{algorithm}
\section{Additional experimental results}
\label{appendix:experiments}
In this section, we provide more experimental details and results. We start by providing those experimental details that remain the same across the datasets. Next, we provide details that are specific to each dataset, i.e., Insurance, Crime, and IHDP. Finally, we provide more experimental results and some discussion.

\subsection{Experimental Details}
In all of our experiments, we use two-layer neural networks. For all hidden layers, we use the \textit{selu} activation function. For the output layer, we use a non-linear activation function only for the variance-prediction network associated with Algorithm \ref{alg:two_stage} to ensure that the predictions of  variance are non-negative. In particular, we use the \textit{soft-plus} activation function for the variance-prediction network associated with Algorithm \ref{alg:two_stage}. In our implementation of Algorithm \ref{alg:suff}, we predict log-variance instead of variance and therefore stick to \textit{linear} activation function.

We train all our neural networks with the Adam optimizer, a batch size of 128, and over 40 epochs. We use a step learning rate scheduler with an initial learning rate of $5\times10^{-3}$ and decay it by a factor of half after every two epochs. As described in Section \ref{sec:expts}, we set the regularizer $\lambda = 1$ for all our experiments after observing that the performance of our algorithm is agnostic to the choice of $\lambda$ as long as it is in a reasonable range, i.e., $\lambda \in [0.5, 3]$.

\subsection{Insurance}  The Insurance dataset\footnote[5]{{\footnotesize https://github.com/stedy/Machine-Learning-with-R-datasets/blob/master/insurance.csv}} is a semi-synthetic dataset that was created using demographic statistics from the U.S. Census Bureau and approximately reflects real-world conditions. A few features in this dataset include the BMI, number of children, age, etc. We remove the sensitive attribute from the set of input features to preprocess the data. To reflect the real-world scenarios where the accuracy disparity is significant due to the small and imbalanced dataset, similar to \cite{chi2021understanding}, we randomly drop 50\% of examples with $D\! =\! 1$. Further, we normalize the output annual medical expenses and the features: age and BMI. We use 3 neurons in the hidden layer for this dataset.

\subsection{Communities and Crime} The Communities and Crime dataset\footnote[6]{{\footnotesize https://archive.ics.uci.edu/ml/datasets/communities+and+crime}} contains socio-economic information of communities in U.S. and their crime rates. 
A few features in this dataset include population for community, mean people per household, percentage of the population that is white, per capita income, number of police cars, etc. We remove the non-predictive attributes and the sensitive attribute from the set of input features during preprocessing. All attributes in the dataset have been curated and normalized to [0, 1], so we don't perform any additional normalization. Finally, we replace the missing values with the mean values of the corresponding attributes similar to \cite{chi2021understanding}. We use 50 neurons in the hidden layer for this dataset.


\subsection{IHDP} The IHDP dataset\footnote[7]{{\footnotesize https://github.com/AMLab-Amsterdam/CEVAE/tree/master/datasets/IHDP}} is generated based on a randomized control trial targeting low-birth-weight, premature infants. The 25 features measure various aspects about the children and their mothers, e.g., child's birth weight, child's gender, mother's age, mother's education, an indicator for maternal alcohol consumption during pregnancy, etc. We remove the sensitive attribute from the set of input features to preprocess the data. Further, we normalize the output cognitive test score and the features: child's birth weight, child's head circumference at birth, number of weeks pre-term that the child was born, birth order, neo-natal health index, and mom's age when she gave birth to the child. Following the norm in the causal inference community, a biased subset of the treated group is removed to create an imbalance leaving 139 samples in the treatment group and 608 samples in the control group. The target is typically simulated using the setting ``A'' of the NPCI package \citep{Dorie2016}. We use 20 neurons in the hidden layer for this dataset.

\subsection{Overall MSE vs. coverage curves}
In Section \ref{sec:expts}, we compared different algorithms in terms of how well they performed selective regression (i.e., with no consideration of fairness) by looking at the area under MSE vs. coverage curve (AUC). Here, we provide the MSE vs. coverage curves for the Insurance dataset in Figure \ref{fig:insurance_joint}, the Crime dataset in Figure \ref{fig:crime_joint}, the IHDP (control) dataset in Figure \ref{fig:ihdp_c_joint}, and the IHDP (treatment) dataset in Figure \ref{fig:ihdp_t_joint}. 

For the Insurance dataset, we see that Algorithm \ref{alg:suff} and Algorithm \ref{alg:two_stage} perform selective regression better than \texttt{Baseline 1} and \texttt{Baseline 2}, respectively. This is also evident via the values of AUC in Table \ref{table:aucs_app}. For the Crime dataset, the MSE decreases with decrease in coverage as expected for all four algorithms. Further, the performances of \texttt{Baseline 1} and \texttt{Baseline 2} are slightly better than that of Algorithm \ref{alg:suff} and Algorithm \ref{alg:two_stage} respectively. This is also evident via the values of AUC in Table \ref{table:aucs_app}. for the IHDP (control) dataset and IHDP (treatment) dataset, we see that Algorithm \ref{alg:suff} performs selective regression better than \texttt{Baseline 1}. This is also evident via the values of AUC in Table \ref{table:aucs_app}.
\begin{figure*}[t]
\centering
\begin{subfigure}{.4\textwidth}
\includegraphics[width=0.8\textwidth]{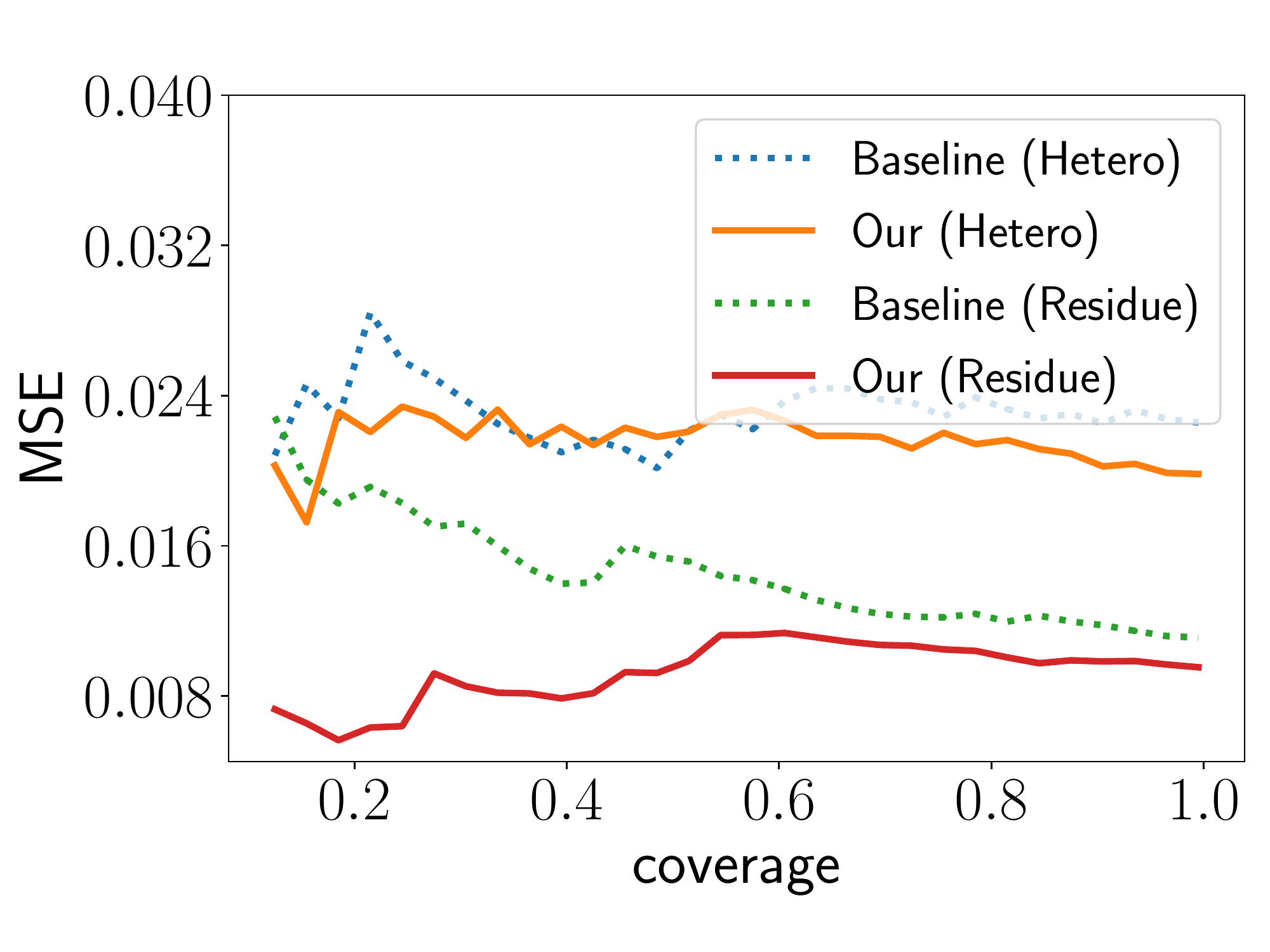}
 \caption{Insurance}
  \label{fig:insurance_joint}
\end{subfigure}
\begin{subfigure}{.4\textwidth}
\includegraphics[width=0.8\textwidth]{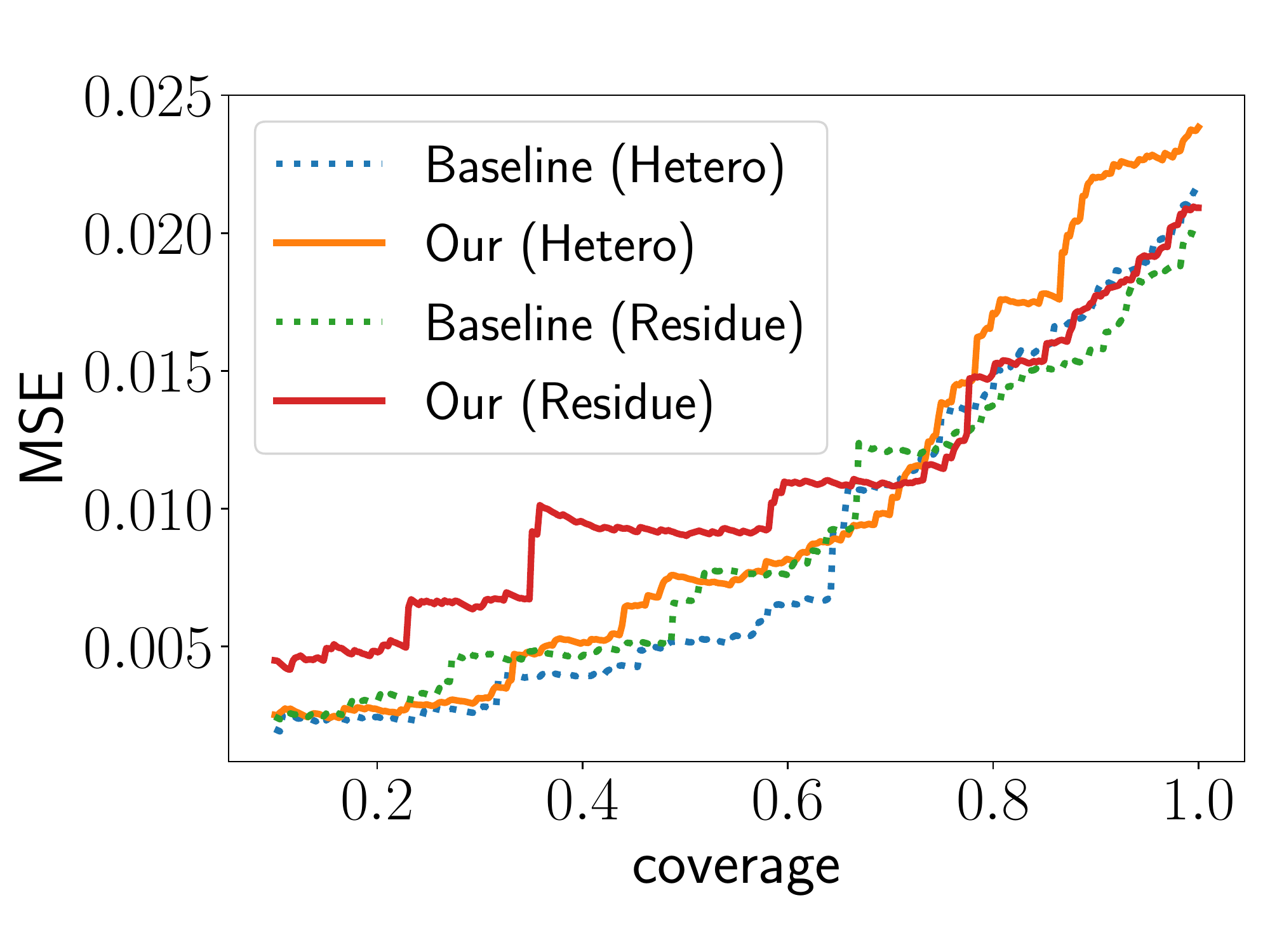}
 \caption{Crime}
  \label{fig:crime_joint}
\end{subfigure}\hfill
\begin{subfigure}{.4\textwidth}
\includegraphics[width=0.8\textwidth]{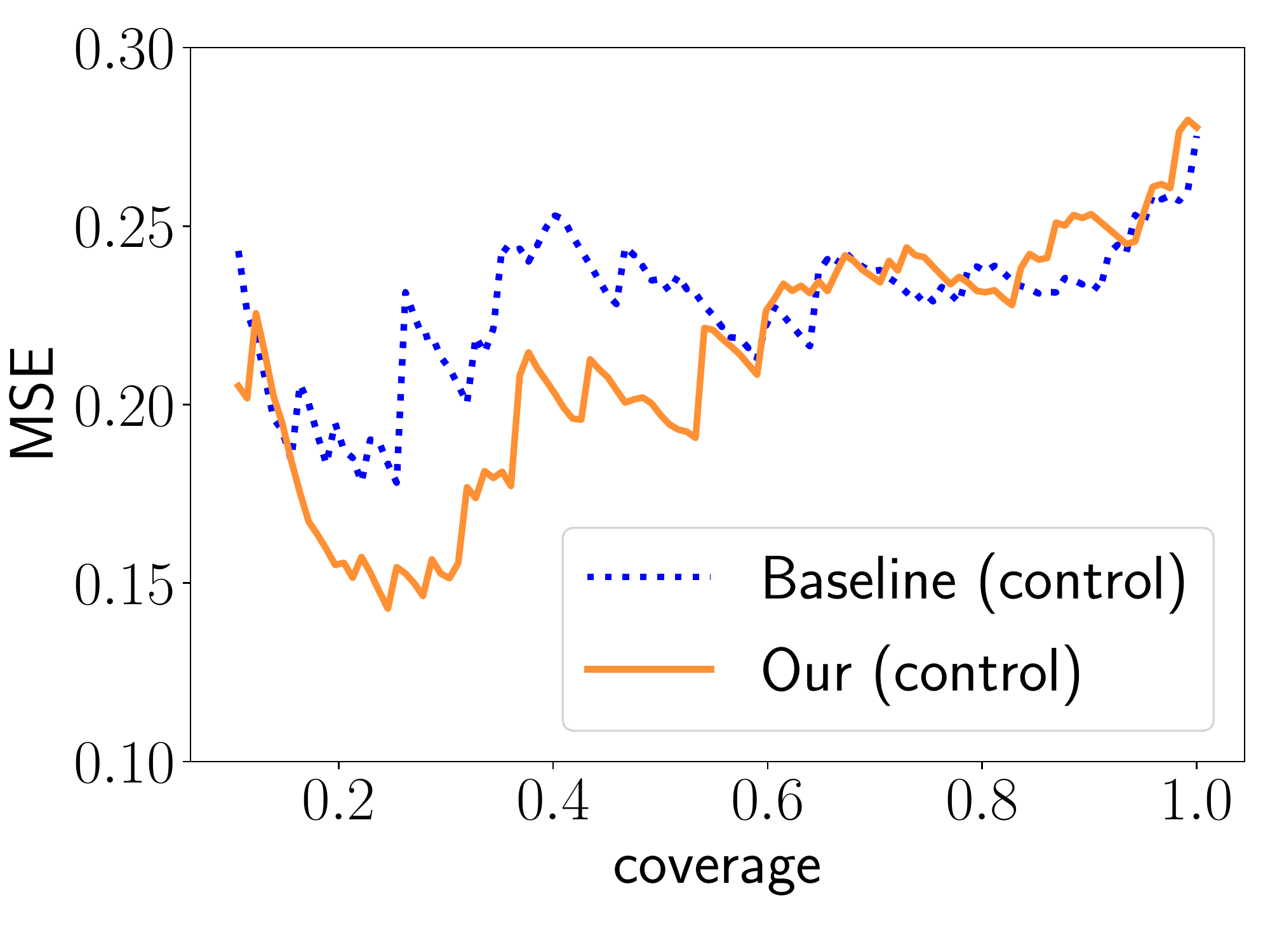}
 \caption{IHDP (control)}
  \label{fig:ihdp_c_joint}
\end{subfigure}%
\begin{subfigure}{.4\textwidth}
\includegraphics[width=0.8\textwidth]{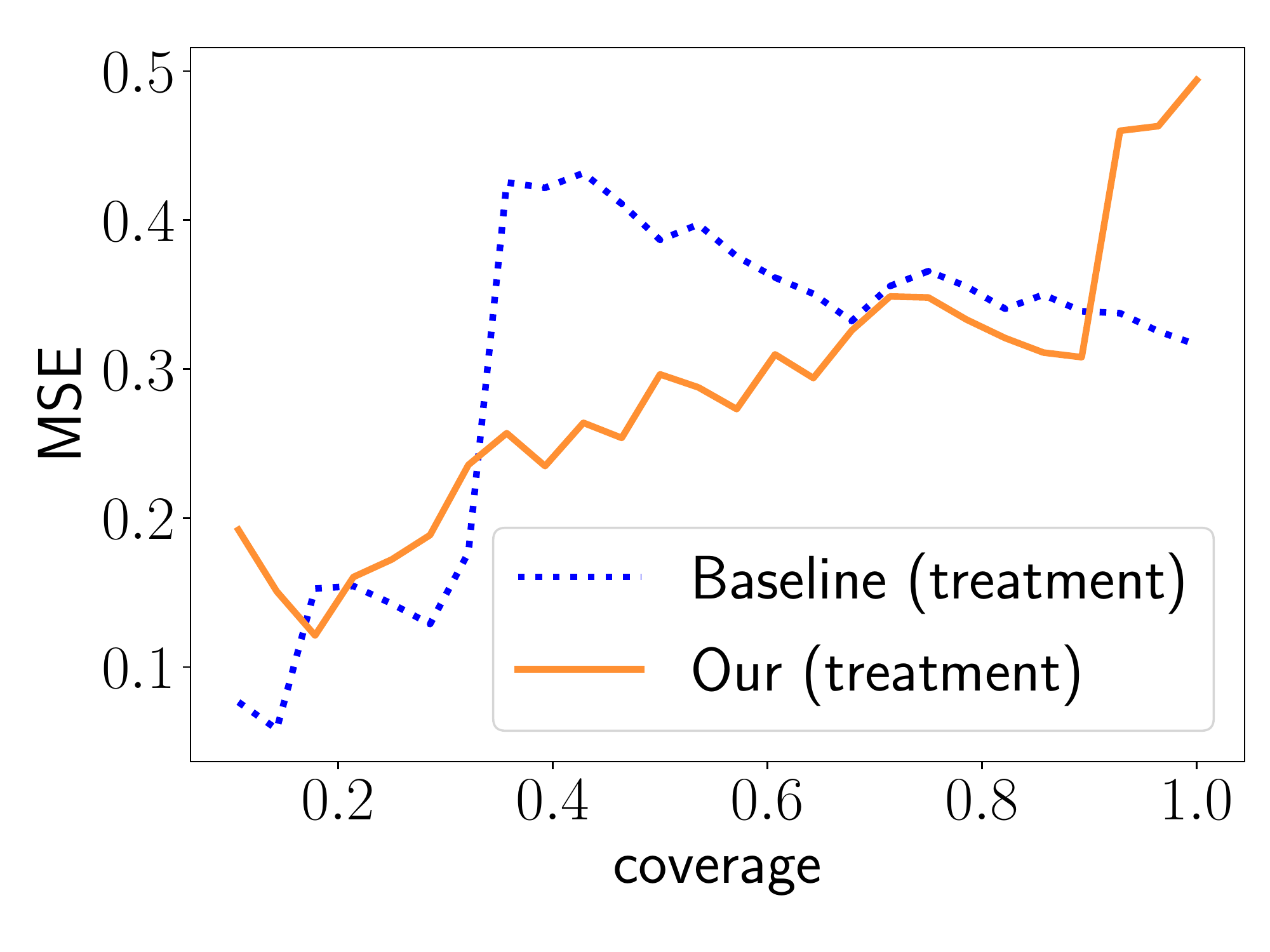}
 \caption{IHDP (treatment)}
  \label{fig:ihdp_t_joint}
\end{subfigure}%
\caption{MSE vs. coverage for various datasets}
\label{fig:joint}
\end{figure*}
\begin{figure*}[h]
\centering
\begin{subfigure}{.33\textwidth}
\centering
  \captionsetup{justification=centering}
  \includegraphics[width=0.95\linewidth]{figures/insurance_baseline_one_stage.pdf}
  \caption{Performance of \texttt{Baseline 1}\\ for the Insurance dataset.}
  \label{fig:insurance_baseline_one}
  \includegraphics[width=0.95\linewidth]{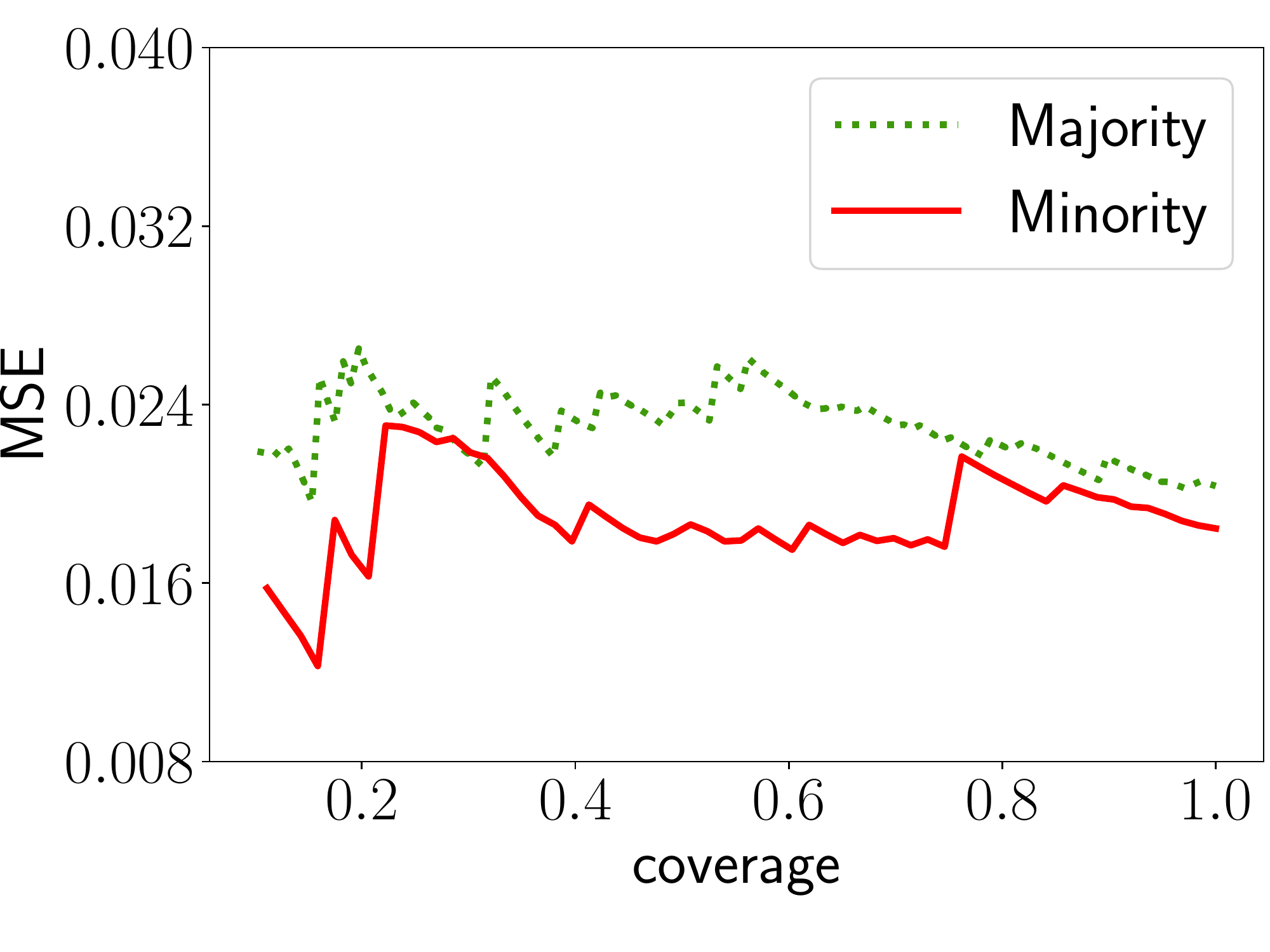}
 \caption{Performance of Algorithm \ref{alg:suff} \\for the Insurance dataset.}
  \label{fig:insurance_suff_one}
\end{subfigure}\hfill
\begin{subfigure}{.33\textwidth}
\centering
  \captionsetup{justification=centering}
  \includegraphics[width=0.95\linewidth]{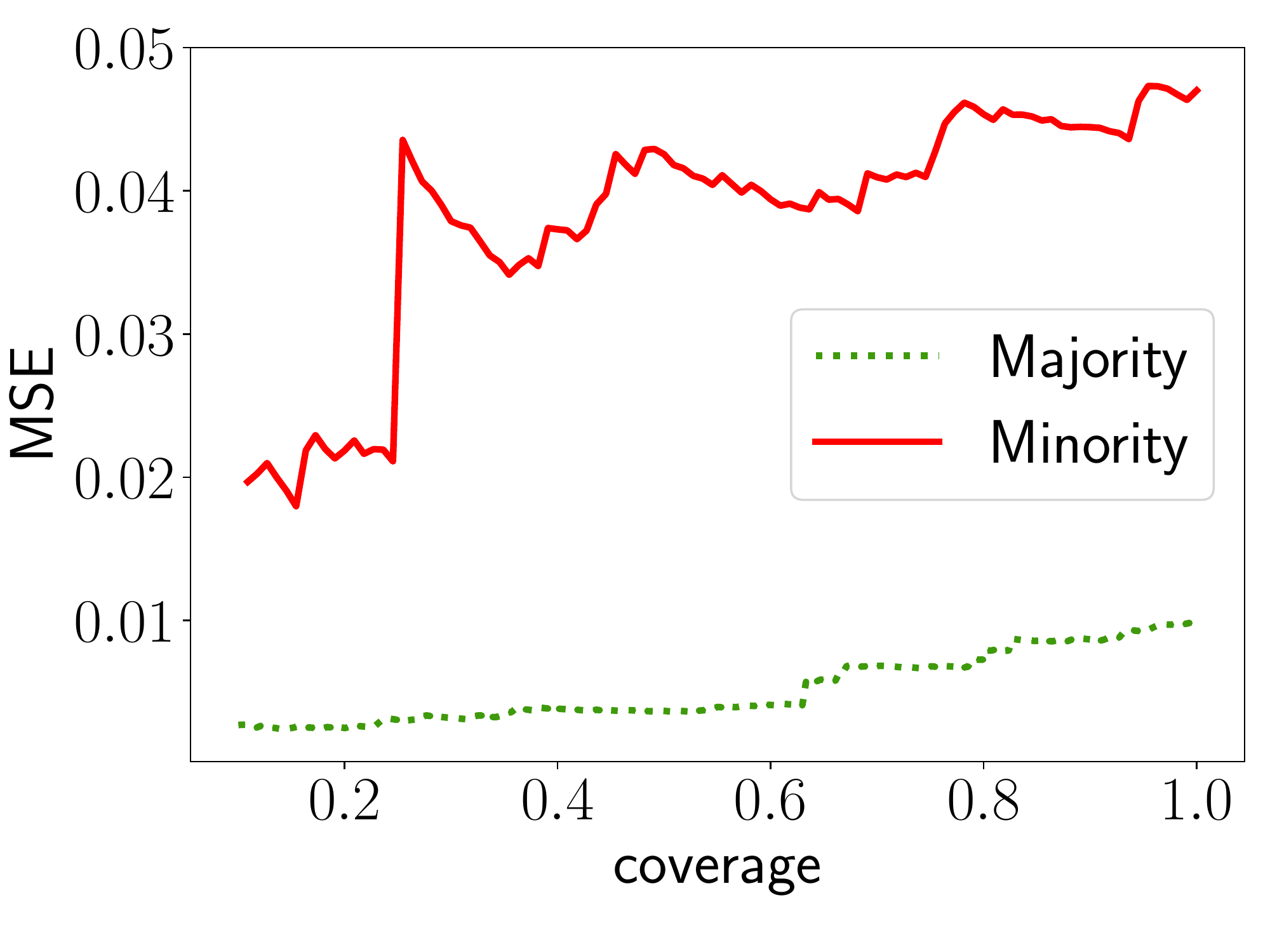}
  \caption{Performance of \texttt{Baseline 2} \\ for the Crime dataset.}
  \label{fig:crime_baseline_two}
  \includegraphics[width=0.95\linewidth]{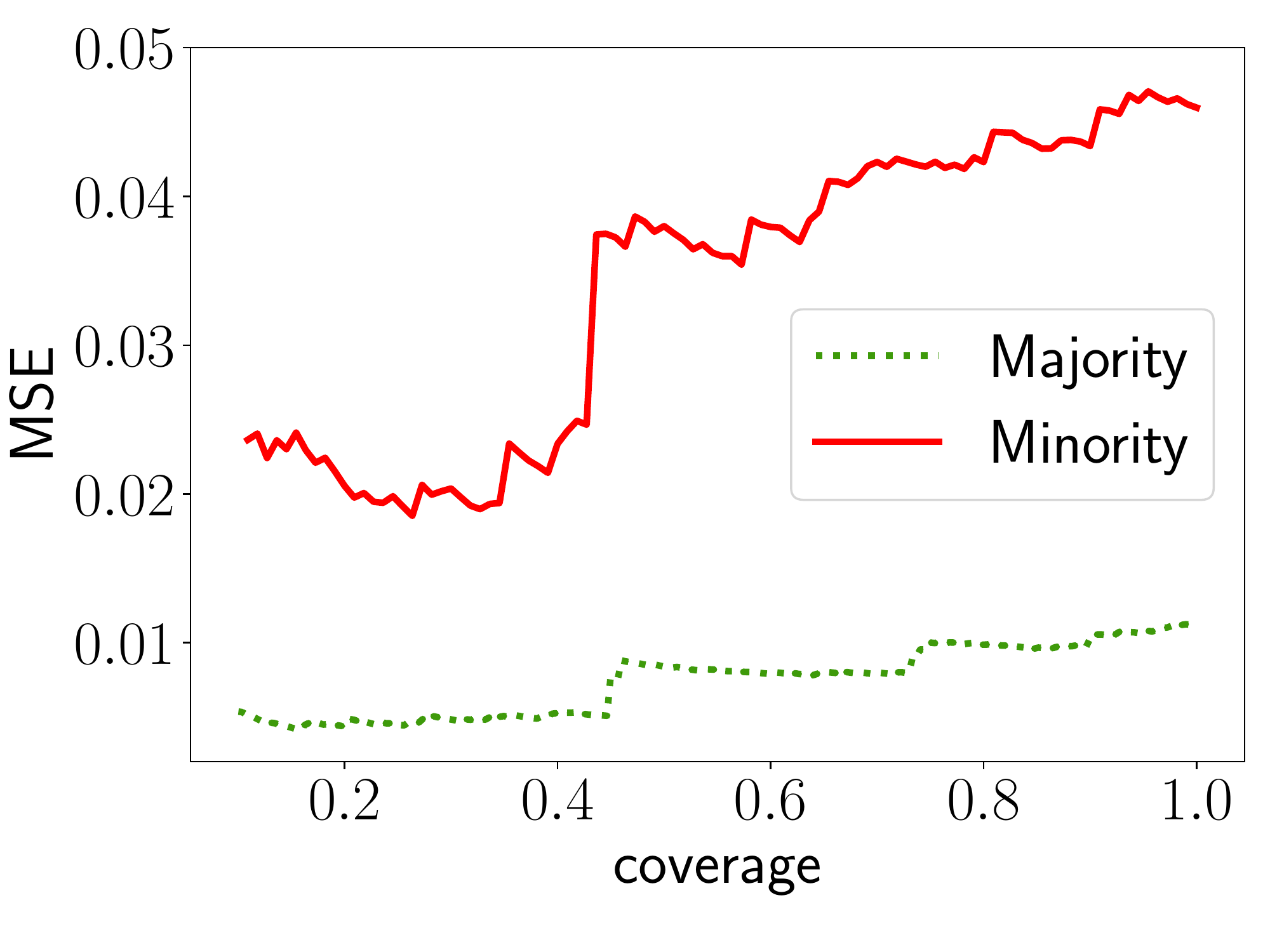}
  \caption{Performance of Algorithm \ref{alg:two_stage} \\ for the Crime dataset.}
  \label{fig:crime_suff_two}
\end{subfigure}\hfill
\begin{subfigure}{.34\textwidth}
\centering
  \captionsetup{justification=centering}
  \includegraphics[width=0.93\linewidth]{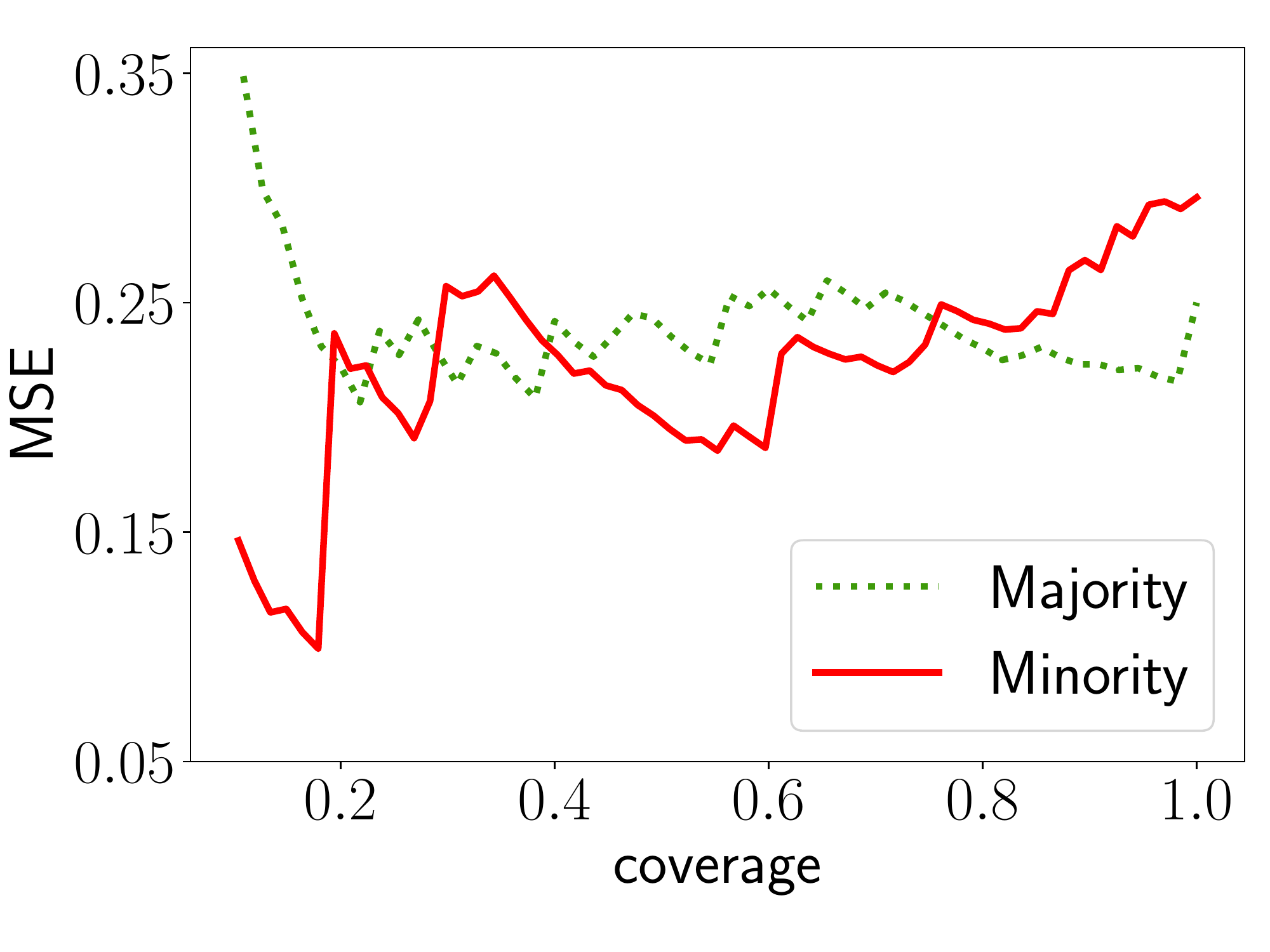}
  \caption{Performance of \texttt{Baseline 1} \\ for the IHDP (control) dataset.}
  \label{fig:ihdp_control_baseline}
  \includegraphics[width=0.93\linewidth]{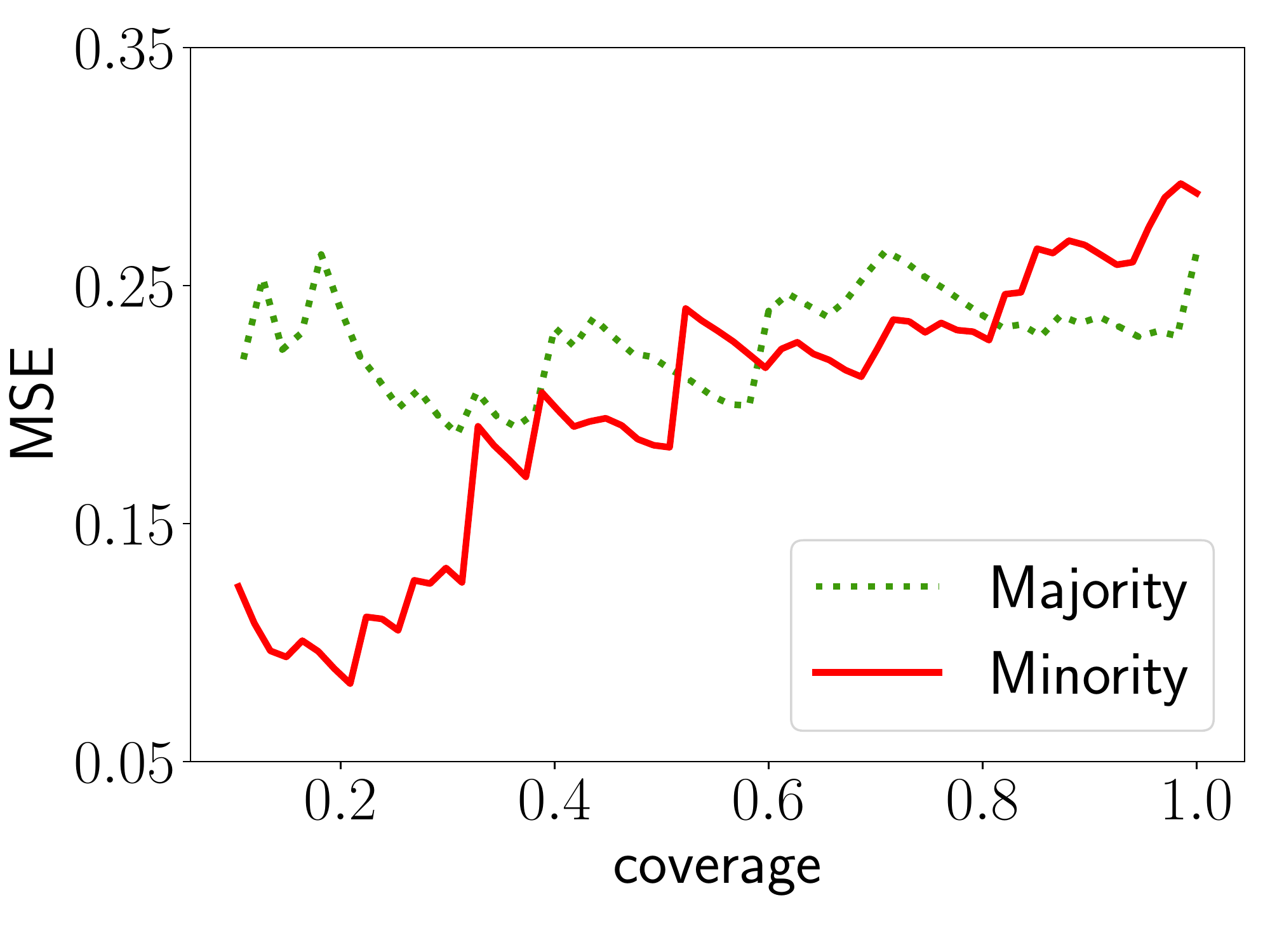}
 \caption{Performance of Algorithm \ref{alg:suff} \\ for the IHDP (control) dataset.}
   \label{fig:ihdp_control_sufficiency}
\end{subfigure}
\caption{Subgroup MSE vs. coverage plots for various datasets comparing baselines (top) and our algorithms (bottom).}
\label{fig:main_2}
\end{figure*}
\subsection{Group-specific MSE vs. coverage curves.}
In Section \ref{sec:expts}, we compared the different algorithms in terms of how well they perform fair selective regression by looking at the subgroup MSE vs. coverage curves in addition to AUC, AUC (D = 0), AUC (D = 1), and AUADC. 

More specifically, we looked at  \texttt{Baseline 2} and Algorithm \ref{alg:two_stage} for the Insurance dataset, \texttt{Baseline 1} and Algorithm \ref{alg:suff} for the Crime dataset, and \texttt{Baseline 1} and Algorithm \ref{alg:suff} for the IHDP (treatment) dataset.
Here, we show the subgroup MSE vs. coverage curves for  \texttt{Baseline 1} and Algorithm \ref{alg:suff} for the Insurance dataset (Figure \ref{fig:insurance_baseline_one} and \ref{fig:insurance_suff_one}), \texttt{Baseline 2} and Algorithm \ref{alg:two_stage} for the Crime dataset (Figure \ref{fig:crime_baseline_two} and \ref{fig:crime_suff_two}), and \texttt{Baseline 1} and Algorithm \ref{alg:suff} for the IHDP (control) dataset (Figure \ref{fig:ihdp_control_baseline} and \ref{fig:ihdp_control_sufficiency}).

For the Insurance dataset, we see that subgroup MSE for the minority subgroup increases with decrease in coverage for \texttt{Baseline 1} (Figure \ref{fig:insurance_baseline_one}) as already described in Section \ref{subsec:biases_sr}. In contrast, the subgroup MSE for Algorithm \ref{alg:suff} does not increase with  decrease in coverage and stays relatively flat (Figure \ref{fig:insurance_suff_one}). Further, for a particular coverage, Algorithm \ref{alg:suff} achieves a better MSE for the minority subgroup, a comparable MSE for the majority subgroup, and reduces the gap between the subgroup curves than \texttt{Baseline 1} (see the values of AUC (D = 0), AUC (D = 1), and AUADC in Table \ref{table:aucs_app}).

For the Crime dataset, we see that the subgroup MSE for \texttt{Baseline 2} as well as Algorithm \ref{alg:two_stage} tends to decrease with a decrease in coverage for both subgroups (Figure \ref{fig:crime_baseline_two} and Figure \ref{fig:crime_suff_two}). However, for a particular coverage, Algorithm \ref{alg:two_stage} achieves a better MSE for the minority subgroup, a comparable MSE for the majority subgroup, and reduces the gap between the subgroup curves than \texttt{Baseline 2} (see the values of AUC (D = 0), AUC (D = 1), and AUADC in Table \ref{table:aucs_app}).

For the IHDP (control) dataset, the subgroup MSE for \texttt{Baseline 1} increases with decrease in coverage (Figure \ref{fig:ihdp_control_baseline}). In contrast, the subgroup MSE for Algorithm \ref{alg:suff} decreases with decrease in coverage (Figure \ref{fig:ihdp_control_sufficiency}). Additionally, Algorithm \ref{alg:suff} achieves a comparable MSE for the majority subgroup, and reduces the gap between the subgroup curves than \texttt{Baseline 1} (see the values of AUC (D = 0), AUC (D = 1), and AUADC in Table \ref{table:aucs_app}).

\subsection{Group-specific MSE vs. coverage curves with three subgroups}
\label{subsec:three_subgroups}
In all of our experiments so far, we focused on the scenario where $D$ was binary. As we now demonstrate, our approach works equally well when $D$ can take more than two values. 

More concretely, we use the Crime dataset to obtain 3 subgroups i.e., the sensitive attribute can take three values. This is possible since, race, the sensitive attribute, is reported as the population percentage of the black in the Crime dataset. We assign (a) $D = 2$ if the population percentage of the black is more than or equal to 20, (b) $D = 1$ if the population percentage of the black is less than 20 but more than or equal to 1, and (c) $D = 0$ otherwise. We show the performance of \texttt{Baseline 1} and Algorithm \ref{alg:suff} in Figure \ref{fig:crime_baseline_3} and \ref{fig:crime_suff_3} respectively. As expected, Algorithm \ref{alg:suff} ensures monotonic selective risk unlike \texttt{Baseline 1} (see $D = 2$).
\begin{figure*}[h]
\centering
\begin{subfigure}{.4\textwidth}
\includegraphics[width=0.8\textwidth]{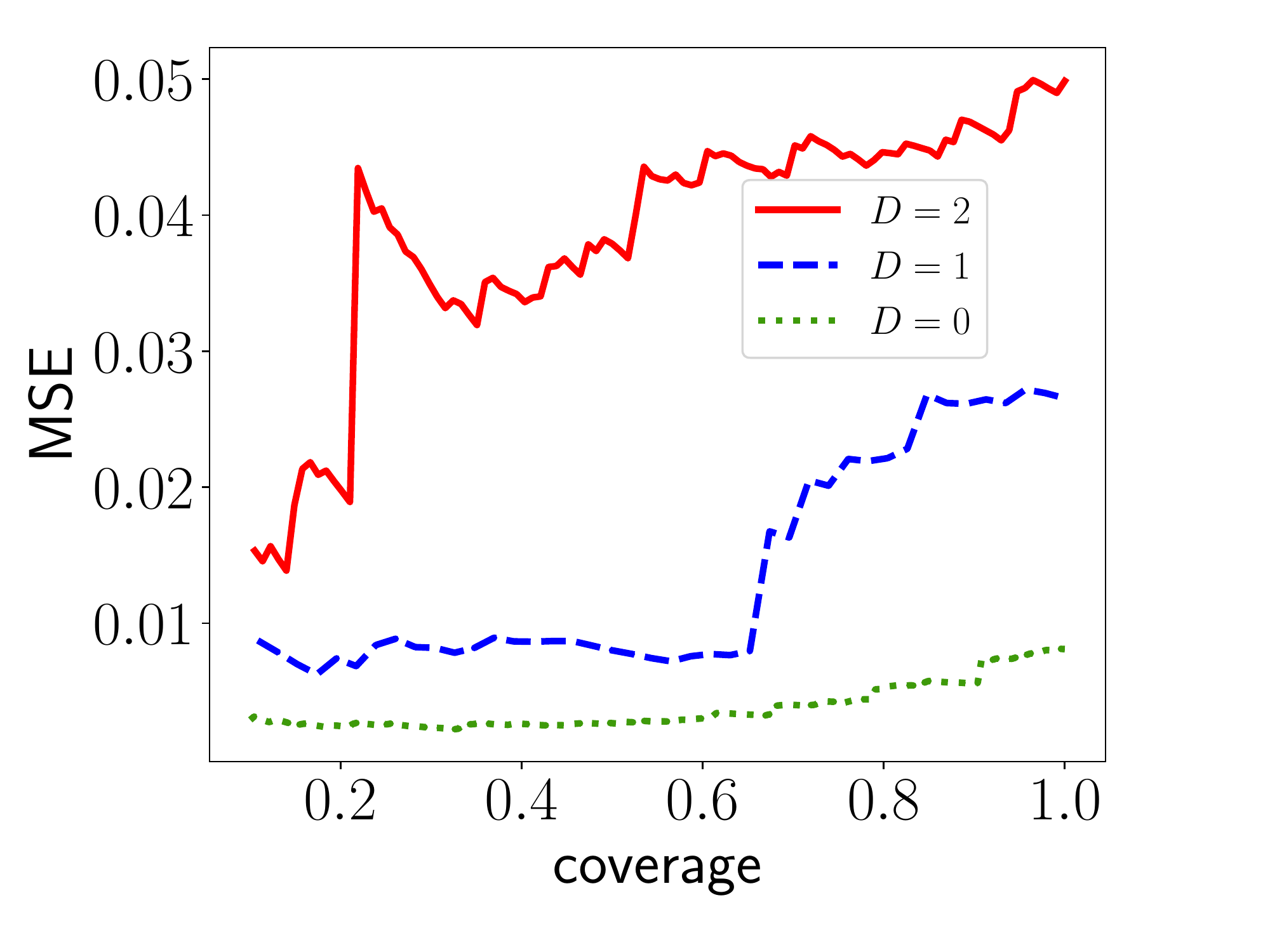}
 \caption{Performance of \texttt{Baseline 1}.}
  \label{fig:crime_baseline_3}
\end{subfigure}
\centering
\begin{subfigure}{.4\textwidth}
\includegraphics[width=0.8\textwidth]{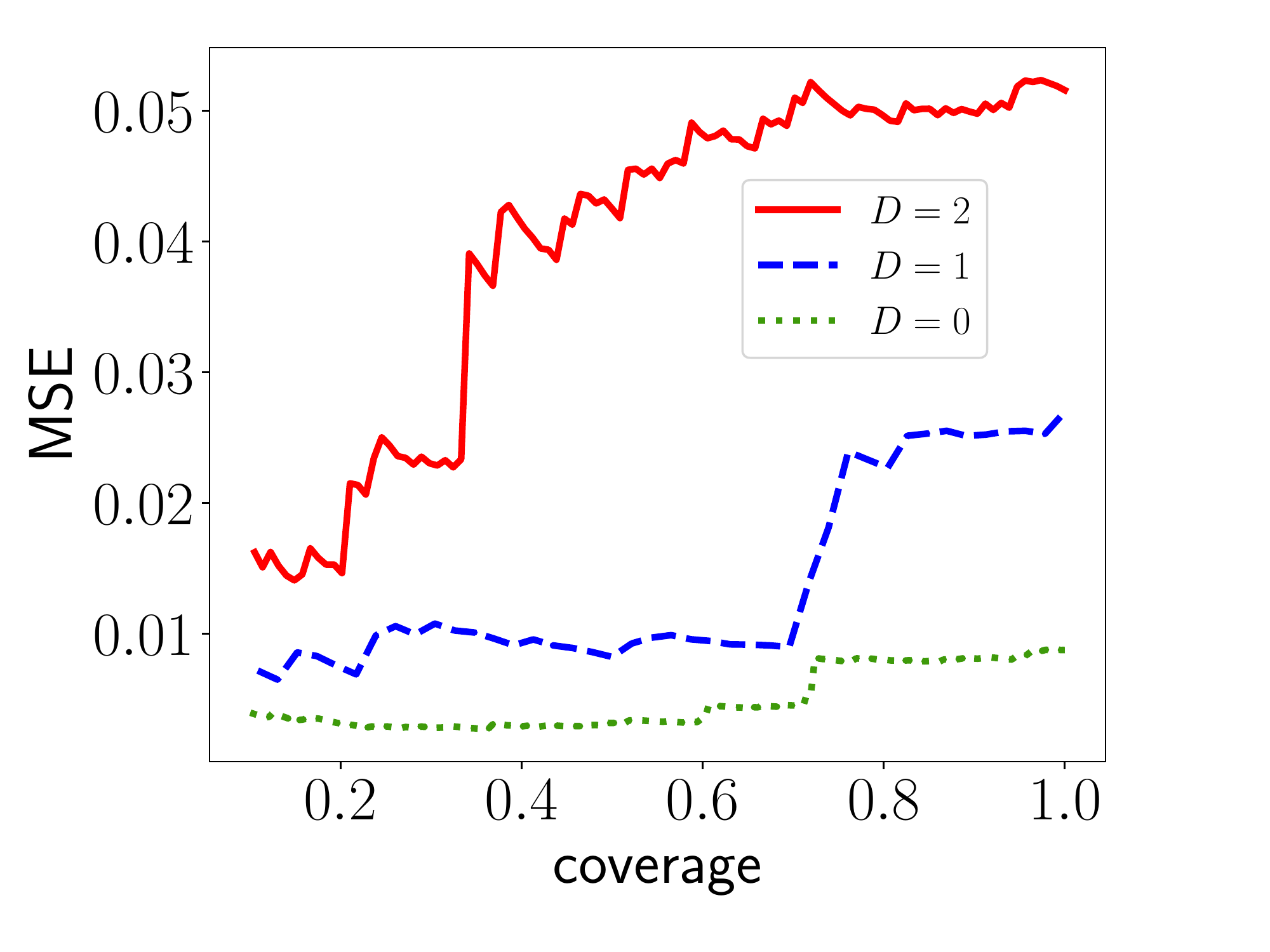}
 \caption{Performance of \texttt{Algorithm 1}.}
  \label{fig:crime_suff_3}
\end{subfigure}
\caption{Subgroup MSE vs. coverage plots for the Crime dataset with three subgroups}
\label{fig:crime_3}
\end{figure*}